%% file: paper.tex
\pdfoutput=1
\documentclass[sigconf]{acmart}
\input{packs}

\input{define}

\copyrightyear{2023} 
\acmYear{2023} 
\setcopyright{acmlicensed}
\acmConference[WWW '23]{Proceedings of the ACM Web Conference 2023}{April 30-May 4, 2023}{Austin, TX, USA}
\acmBooktitle{Proceedings of the ACM Web Conference 2023 (WWW '23), April 30-May 4, 2023, Austin, TX, USA}
\acmPrice{15.00}
\acmDOI{10.1145/3543507.3583370}
\acmISBN{978-1-4503-9416-1/23/04}
\author{Kota Nakamura}
\affiliation{SANKEN, Osaka University, Japan}
\email{kota88@sanken.osaka-u.ac.jp}
\author{Yasuko Matsubara}
\affiliation{SANKEN, Osaka University, Japan}
\email{yasuko@sanken.osaka-u.ac.jp}
\author{Koki Kawabata}
\affiliation{SANKEN, Osaka University, Japan}
\email{koki@sanken.osaka-u.ac.jp}
\author{Yuhei Umeda}
\affiliation{AI Lab., Fujitsu, Japan}
\email{umeda.yuhei@fujitsu.com}
\author{Yuichiro Wada}
\affiliation{AI Lab., Fujitsu; AIP, RIKEN, Japan}
\email{wada.yuichiro@fujitsu.com}
\author{Yasushi Sakurai}
\affiliation{SANKEN, Osaka University, Japan}
\email{yasushi@sanken.osaka-u.ac.jp}
\begin{document}
\title{Fast and 
Multi-aspect Mining 
of \\
Complex Time-stamped Event Streams
}
\begin{abstract}
    \input{000abstract}

\end{abstract}
\maketitle
\section{Introduction}
    \label{010intro}
    \input{010intro}
\section{Related work}
    \label{020related}
    \input{020related}
\section{Proposed model}
    \label{030model}
    \input{030model}

\section{Streaming Algorithm}
    \label{040proposed}
    \input{040proposed}
\section{Experiments}
    \label{050experiments}
    \input{050experiments}
\section{Conclusion}
    \label{060conclusions}
    \input{060conclusions}

\begin{acks}
We sincerely thank the anonymous reviewers, 
for their time and effort during the review process.
This work was supported by
JSPS KAKENHI Grant-in-Aid for Scientific Research Number
JP20H00585,    
JP21H03446,    
JP22K17896,    
NICT 03501, 
MIC/SCOPE JP192107004, 
JST-AIP JPMJCR21U4, 
ERCA-Environment Research and 
Technology Development Fund JPMEERF20201R02. 
\end{acks}

\bibliographystyle{ACM-Reference-Format} 
\bibliography{%
BIB/ref_compression,%
BIB/ref_stream,%
BIB/ref_timeseries,%
BIB/ref_yasuko_self,%
BIB/ref_misc%
}

\clearpage
\appendix
\input{appendix}

\clearpage

\end{document}

%% file: packs.tex
\usepackage{hyperref}
\hypersetup{
    colorlinks=true,
    linkcolor=black,
    urlcolor=black,
    citecolor=black,
    breaklinks=true 
}

\def\equationautorefname~#1\null{Equation~(#1)\null}
\usepackage{breakurl}
\usepackage{color}
\usepackage{bm}
\usepackage{xspace}
\usepackage{algorithmic}
\usepackage{algorithm}
\usepackage{mdwlist}    
\usepackage{paralist} 

\usepackage{utfsym}

\usepackage{adjustbox}
\usepackage{array}
\usepackage{booktabs}
\usepackage{multirow}

\usepackage{array,graphicx}
\usepackage{booktabs}
\usepackage{pifont}
\usepackage{color}

\usepackage{colortbl}
\definecolor{lightgray}{gray}{0.85}
\usepackage{multirow} 

\usepackage{fancybox} 
\usepackage{amsmath}
\usepackage{amsfonts}


%% file: define.tex
\newcommand{\TSK}[1]{#1}

\newcommand{\vswd}{\vspace{0.0em}}
\newcommand{\bit}{\vspace{-0em}\begin{itemize}}
\newcommand{\eit}{\end{itemize}\vspace{-0.2em}}
\newcommand{\ben}{\vspace{-0em}\begin{enumerate}}
\newcommand{\een}{\end{enumerate}\vspace{-0.2em}}
\newcommand{\bea}{\vspace{-0em}\begin{eqnarray}}
\newcommand{\eea}{\end{eqnarray}\vspace{-0.0em}}
\newcommand{\beq}{\vspace{-0.0em}\begin{equation}}
\newcommand{\eeq}{\end{equation}\vspace{-0.0em}}

\renewcommand{\bit}{\vswd\begin{compactitem}}
\renewcommand{\eit}{\end{compactitem}\vswd}
\renewcommand{\ben}{\vswd\begin{compactenum}}
\renewcommand{\een}{\end{compactenum}\vswd}
    
\newcommand{\argmax}{\mathop{\rm arg~max}\limits}
\newcommand{\argmin}{\mathop{\rm arg~min}\limits}

\newcommand{\shapeN}{\in\mathbb{N}}  

\newcommand{\Ex}{{x}}

\newcommand{\tensor}{\mathcal{X}}

\newcommand{\tensorC}{\mathcal{X}^{C}}

\newcommand{\nfactor}[1]{r_1}

\newcommand{\segmentset}{\mathcal{S}}

\newcommand{\score}[1]{score(#1)}
\newcommand{\norm}{norm}

\newcommand{\hide}[1]{}
\newcommand{\mypara}[1]{\vspace{1.00em}\noindent\textbf{#1.}}

\newcommand{\myparaitemize}[1]{\noindent{\textbf{#1.}}}

\newcommand{\prop}[1]{({\bf{#1}})\xspace}

\newcommand{\myiparapara}[1]{\vspace{0.00em}{\em #1:}}

\newtheorem{observation}{Observation}
\newtheorem{informalProblem}{InformalProblem}
\newtheorem{problem}{Problem}
\newtheorem{definition}{Definition}

\newcommand{\attribute}{attribute\xspace}
\newcommand{\attributes}{attributes\xspace}

\newcommand{\unit}{unit\xspace}
\newcommand{\units}{units\xspace}

\newcommand{\topic}{component\xspace}
\newcommand{\topics}{components\xspace}
\newcommand{\Topic}{Component\xspace}

\newcommand{\timestamp}{time\xspace}
\newcommand{\method}{\textsc{CubeScope}\xspace}
 
\newcommand{\Etstream}{Event tensor stream\xspace}
\newcommand{\etstream}{event tensor stream\xspace}

\newcommand{\Matt}{\mathbf{A}}
\newcommand{\Mtime}{\mathbf{B}}

\newcommand{\Patt}{\mathbf{\hat{A}}}
\newcommand{\Ptime}{\mathbf{\hat{B}}}

\newcommand{\Eatt}{{a}}
\newcommand{\Etime}{{b}}

\newcommand{\patt}{\hat{a}}
\newcommand{\ptime}{\hat{b}}

\newcommand{\regime}{\theta}
\newcommand{\regimeset}{\Theta}
\newcommand{\regimeassignment}{\mathcal{S}}
\newcommand{\regimeassign}{s}
\newcommand{\cand}{\mathcal{C}}

\newcommand{\alt}{e}

\newcommand{\costM}[1]{<#1>}
\newcommand{\costC}[2]{<#1|#2>}
\newcommand{\costT}[2]{<#1;#2>}
\newcommand{\cF}{c^F}

\newcommand{\nunits}{U}
\newcommand{\nmode}{M}
\newcommand{\ntopic}{K}
\newcommand{\duration}{T}
\newcommand{\nregime}{R}
\newcommand{\nshiftp}{G}

\newcommand{\lunit}{u}
\newcommand{\lmode}{m}
\newcommand{\ltopic}{k}
\newcommand{\ltime}{t}
\newcommand{\lregime}{r}
\newcommand{\lshiftp}{g}

\newcommand{\ntotal}{N}	
\newcommand{\mathN}{\mathbb{N}}
\newcommand{\mathR}{\mathbb{R}}

\newcommand{\methodalgo}{{\textsc{CubeScope}}\xspace}
\newcommand{\decomposition}{{\textsc{C-Decomposer}}\xspace}
\newcommand{\compression}{{\textsc{C-Compressor}}\xspace}

\newcommand{\ytaxi}{\textit{NYC-Taxi}\xspace}
\newcommand{\jewelry}{\textit{Jewelry}\xspace}
\newcommand{\electronics}{\textit{Electronics}\xspace}
\newcommand{\nybike}{\textit{Bike-Share}\xspace}
\newcommand{\kddcup}{\textit{AirForce}\xspace}
\newcommand{\ciexternal}{\textit{External}\xspace}
\newcommand{\ciinternal}{\textit{OpenStack}\xspace}
\newcommand{\kyoto}{\textit{Kyoto}\xspace}

%% file: 000abstract.tex
Given a
huge,
online stream of time-evolving events with multiple attributes,
such as online shopping logs: 
{\it(item, price, brand, \timestamp)},
how can we summarize large, dynamic high-order tensor streams? 
How can we see any hidden patterns, rules,
and 
anomalies?
Our answer is to focus on
two types of patterns, 
i.e., ``regimes'' and ``components'',
over high-order tensor streams,
for which we
present an efficient and effective method, 
namely \method.
Specifically, 
it identifies any sudden discontinuity 
and recognizes distinct dynamical patterns, ``regimes''
(e.g., weekday/weekend/holiday patterns).
In each regime, 
it also performs multi-way summarization for all attributes 
(e.g., item, price, brand, and \timestamp) 
and discovers hidden ``\topics''
representing
latent groups
(e.g., item/brand  groups)
and their relationship.
Thanks to its concise but effective summarization,
\method 
can also
detect the sudden appearance of anomalies 
and identify the types of anomalies that occur in practice.

Our 
proposed method has the following properties:
(a) {\it Effective:}
it
captures
dynamical multi-aspect patterns,
i.e., regimes and \topics,
and
statistically summarizes
all the events;
(b) {\it General:}
it is practical
for successful application to
data compression, 
pattern discovery, and anomaly detection on 
various types of tensor streams;
(c) {\it Scalable:}
our algorithm 
does not depend on the length of the data stream and its dimensionality.
Extensive experiments on 
real datasets demonstrate that 
\method finds meaningful patterns and anomalies correctly, 
and consistently outperforms the state-of-the-art methods as regards accuracy and execution speed.

%% file: 010intro.tex
\TSK{
\begin{figure*}[t]
        \vspace{-1.8em}
     \centering
  \includegraphics[width=\linewidth]{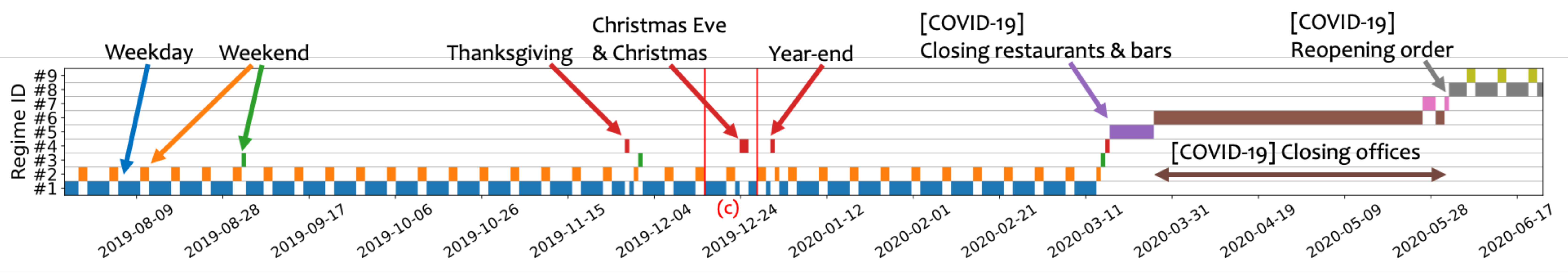}\\
     \vspace{-0.9em}
     (a) Regime identification for social mobility event streams 
     \begin{tabular}{ccc|c}
     \hspace{-1em}
     \begin{minipage}{0.2\linewidth}
     \centering
     \includegraphics[width=0.75\linewidth]{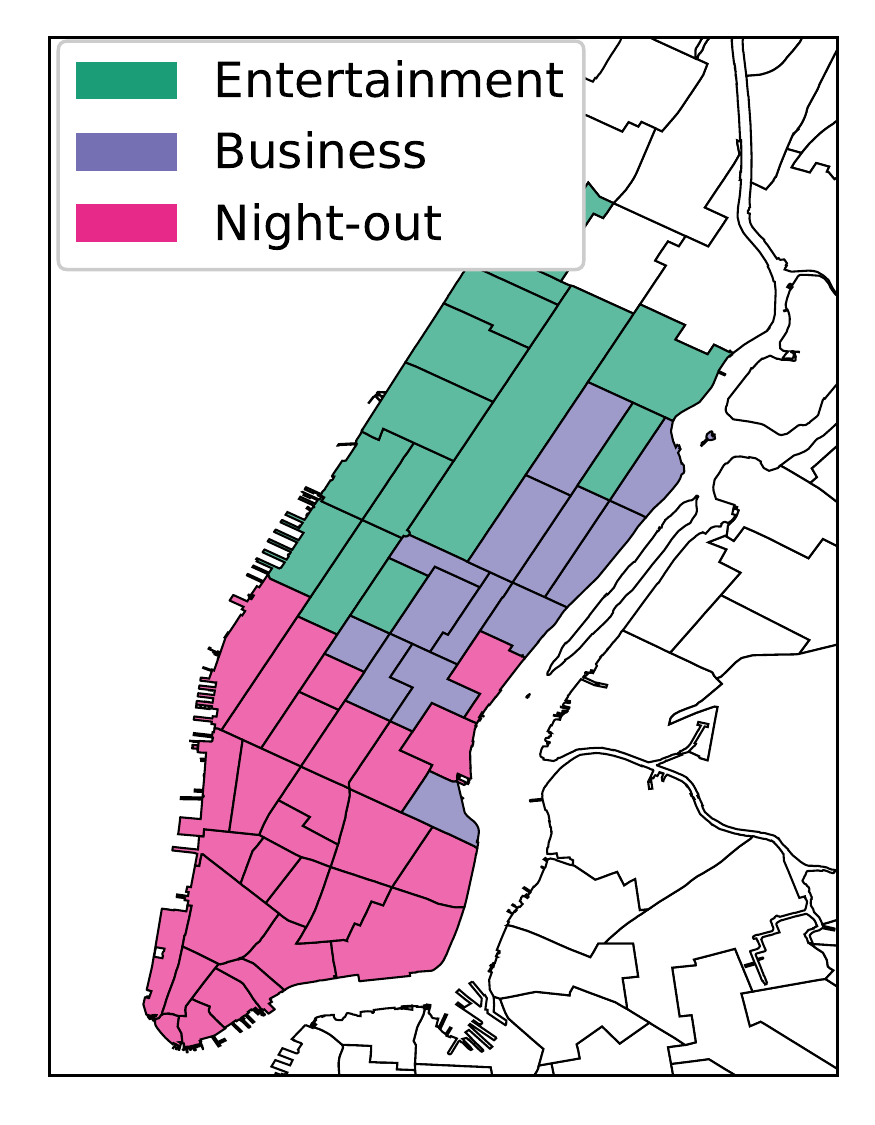}\\
     (b-i) Snapshot of \\
     three \topics \\
     on pick-up location  
     \end{minipage}
     &  
     \hspace{-4em}
     \begin{minipage}{0.2\linewidth}
     \centering
     \includegraphics[width=0.75\linewidth]{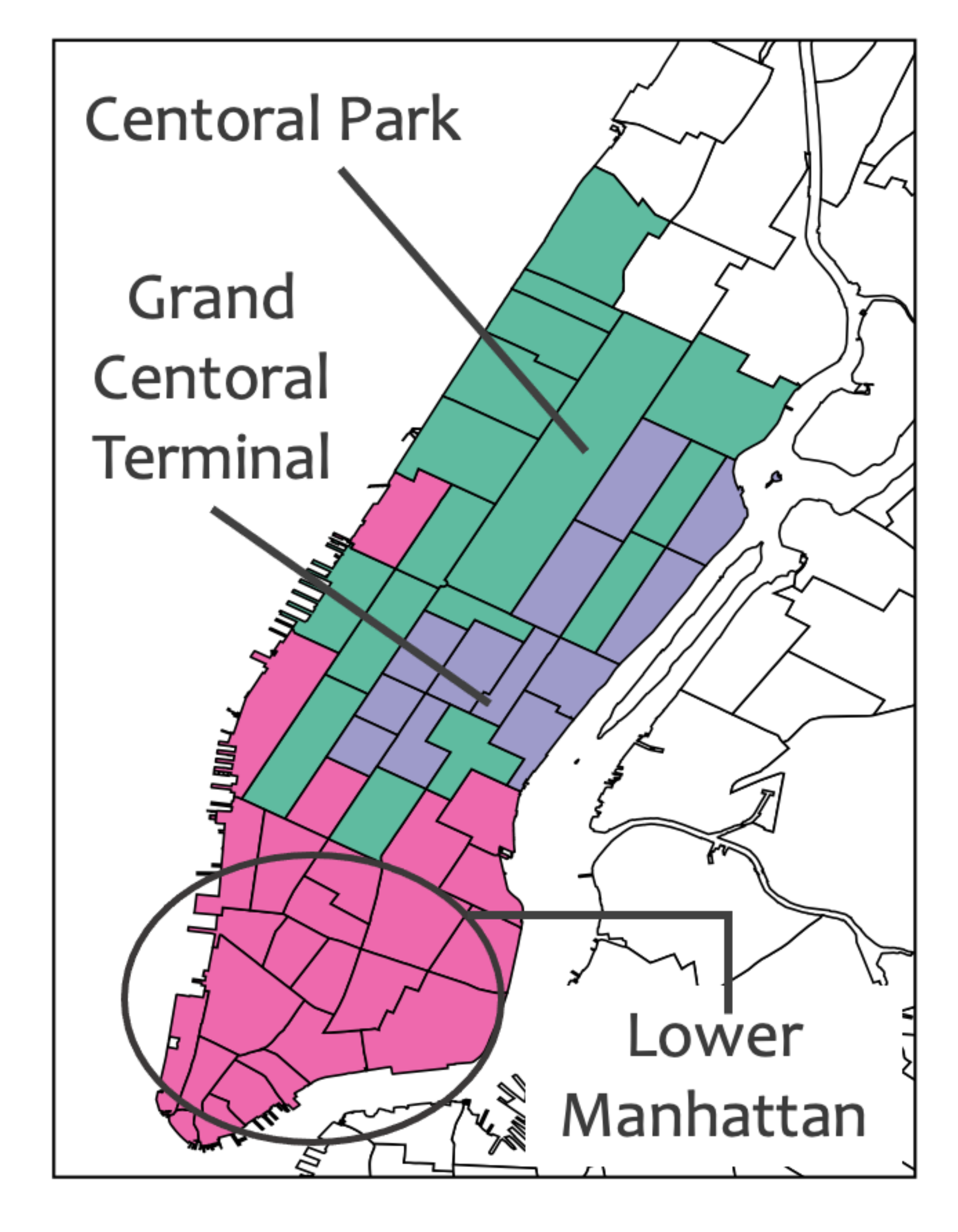} \\
     (b-ii) Snapshot of \\
     three \topics \\
     on drop-off location
     \end{minipage}
 	 \vspace{-0.60em}
 	 \hspace{-4.5em}
     &
     \begin{minipage}{0.45\linewidth}
     \vspace{-0.3em}
     \centering
     \includegraphics[width=0.85\linewidth]{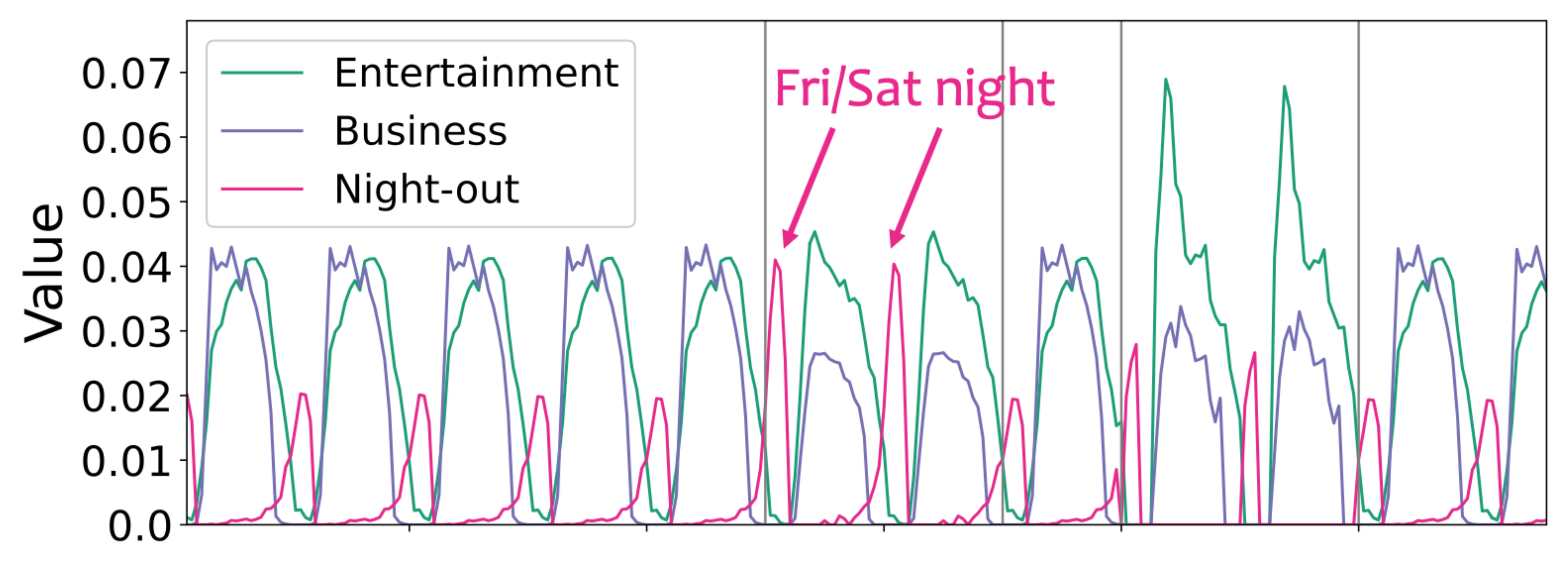}
     \includegraphics[width=0.85\linewidth]{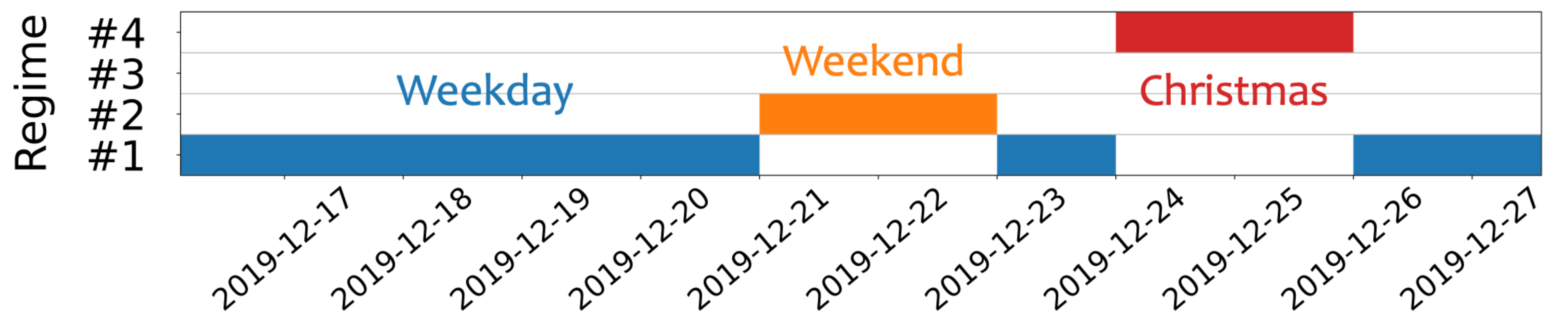}
     \\
     \vspace{-0.4em}
     (c) Time evolution of three \topics \\
     around mid-December     
    \end{minipage}
    \hspace{-1.5em}
     &
    \hspace{-1em}
     \begin{minipage}{0.33\linewidth}
        \vspace{-1em}
         \centering
         \includegraphics[width=0.9\columnwidth]{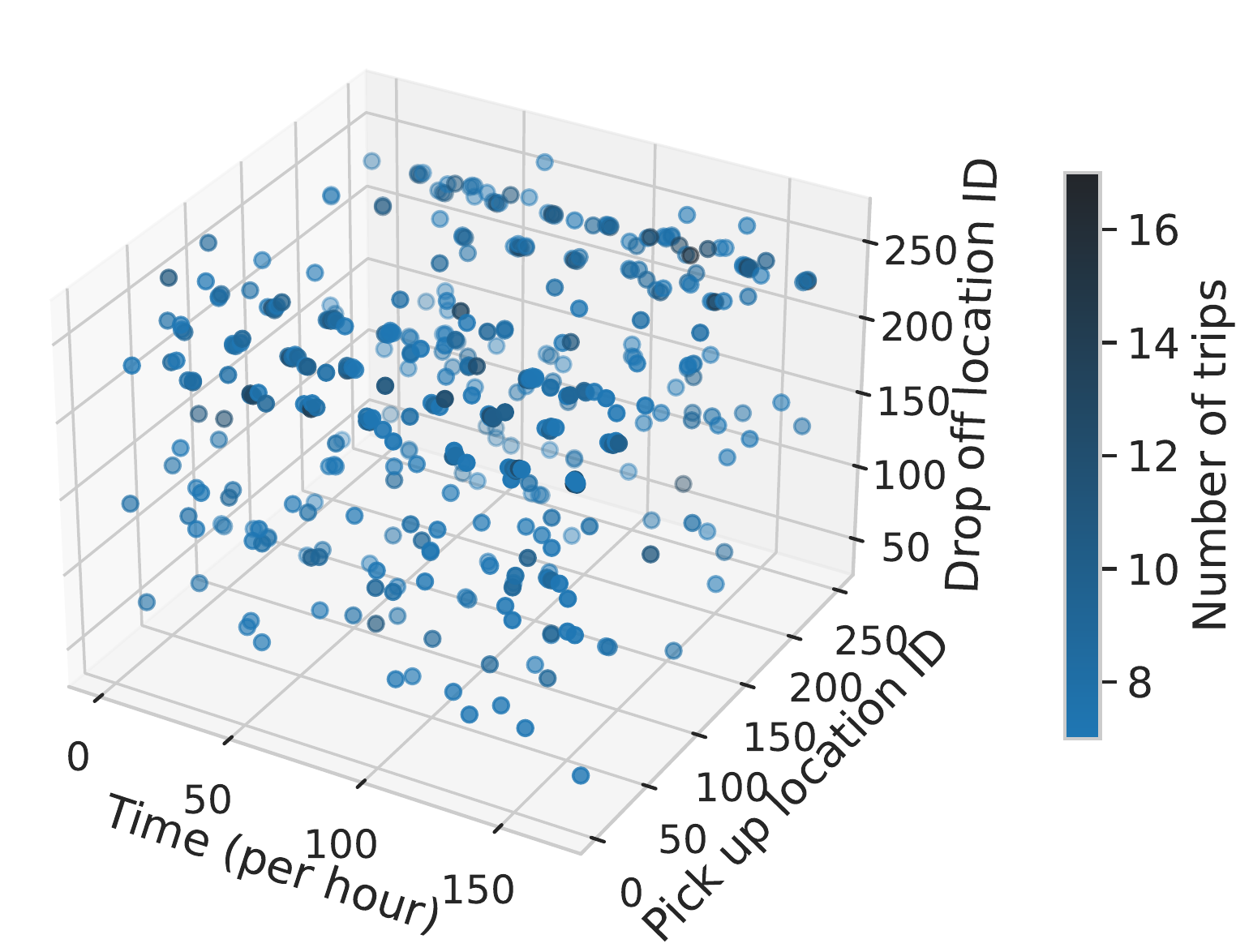}\\
         \hspace{-2em}
         (d) Original event data stream \\  
         \hspace{-3em}
	     for NYC taxi rides (one week)
         \end{minipage}
     \end{tabular}
     \vspace{-0.5em}
     \caption{
        Real-time modeling of \method on New York City taxi rides:
        (a) It incrementally identifies distinct 
        time-evolving 
        patterns (i.e., regimes) and their shifting points.
        Specifically, 
        Regime \#1 (blue) coincides with weekdays, 
        while
        Regimes \#2, \#3, and \#4 
        (orange, green, and red)
        capture weekends and public holidays.
        Also, it can adaptively recognize 
        the sudden regime transitions (Regimes \#5, \#6, $\cdots$)
        that reflect social conditions under the COVID-19 pandemic.
        It finds 
        \topics that are
        interpretable
        summaries 
        for all attributes (i.e., pick-up, drop-off, \timestamp),
        especially
        for 
        (b) pick-up/drop-off locations 
        and 
        (c) time \attribute .
        (d)
        The original data 
        is a sparse and high-dimensional tensor.
        It exhibits no obvious \topics or regimes.
    }
    \vspace{-1.5em}
    \label{fig:preview}
\end{figure*}
}

Given a large, online stream of time-stamped events,
how can we statistically summarize all the event streams 
and find important patterns, 
rules, 
and 
anomalies?
Time-stamped event data
are generated and collected by many real applications
\cite{de2016iot,DBLP:conf/www/LiuSWTSPL19,wang2015rubik,DBLP:conf/kdd/BaytasXZWJZ17},
including
online marketing analytics \cite{DBLP:conf/www/SakuraiMF16,kraus2019personalized},
social network/location-based services \cite{DBLP:conf/kdd/ChoML11,DBLP:conf/kdd/OkawaIK0TU19},
and cybersecurity systems \cite{DBLP:conf/www/0001J0KH21,shin2017densealert},
with increasingly larger sizes and faster rates of transactions.
For example, 
an online shopping service 
could generate millions of logging entries every second, 
with rich information about items and users.
The service providers
would like to 
send
targeted advertisements
and
detect 
fraudulent activities
by investigating 
online purchasing patterns and hidden user/item relationships.

Here, let us 
assume that we have a large collection of event logs,
consisting of multiple attributes, 
e.g., online shopping: 
{\it (item, price, brand, \timestamp)} 
and local mobility activities: 
{\it (pick-up and drop-off locations, \timestamp)},
where huge numbers of event entries 
arrive online at high bit rates,
which 
we shall refer to 
as
``\textit{complex time-stamped event streams}''. 
%
These data are represented as high-order 
tensor streams,
e.g., a $4$th-order item-price-brand-time tensor stream,
unlike 
a previously considered 
multivariate time series \cite{hallac2017toeplitz},
tensor \cite{kolda2009tensor},
or stream of elements 
\cite{manzoor2018xstream},
specifically, 
as mentioned later in Section\ref{sec:designphilo},
whose high-dimensional, sparse, and semi-infinite 
nature
derails existing methods and 
even our interpretation of data.
%
\textit{So what is a good representation of complex
time-stamped event streams?}
This is exactly the problem we focus on in this work.
We first present
a compact yet powerful
representation
that summarizes 
a semi-infinite collection of tensor streams.
Specifically,
we 
aim to capture
two types of patterns, 
i.e., {\it ``regimes''} and {\it ``components''}.

In practice,
real-life data streams 
contain
various types of distinct temporal dynamical patterns 
of different durations,
namely, {\it ``regimes''},
such as the weekday/weekend/holiday patterns
of online shopping services or taxi rides.
%
In each regime,
a set of events, consisting of multiple attributes,
has similar 
behavior
and latent interactions.
We introduce the concept of latent {\it ``\topics''},
which capture hidden groups 
in each \attribute
(e.g., item groups and typical pick-up locations) 
and their relationships.

An important application scenario,
for example, 
in cybersecurity, 
multiple types of intrusions/anomalies,
such as denial of service or port scanning attacks,
occur suddenly
and
need to be detected and analyzed
in real-time
to minimize harm.
So we would also like to answer the question:
\textit{
How can we quickly detect anomalies 
and identify their types?
}
However, it is extremely challenging
because 
signs of 
the anomalies
appear in one or more attributes 
(e.g., source IP address, packet size,$\ldots$),
and 
event streams evolve over time,
where new types of anomalies can arise
and the concept of normal behavior changes.

In this paper, 
we present \method ,
an efficient and effective mining approach
capable of dealing with the above questions.
\method monitors a high-order tensor stream and incrementally 
recognizes
dynamical multi-aspect patterns,
i.e., regimes and \topics, 
and anomalies,
while updating the information for each.
Intuitively, 
the problem we wish to solve is as follows:
\begin{informalProblem}
\textbf {Given}
a high-order tensor stream
$\tensor$,
which consists of events with 
multiple \attributes and timestamps,
\bit
\item
\textbf {Find}
a compact description of $\tensor$ that summarizes all events,
\bit
\item 
distinct dynamical patterns
(i.e., regimes),
\item 
multi-aspect latent trends
(i.e., \topics),
\eit
\item
\textbf{Report} anomalies and their types
\eit
incrementally and quickly, at any point in time.
 
\end{informalProblem}

\myparaitemize{Preview of Results}
\autoref{fig:preview}~(a)-(c) shows 
some of our discoveries 
on local mobility data.
This dataset consists of taxi ride events 
{\it(pick-up location ID, drop-off location ID, \timestamp)}
in New York City, 
with hourly timestamps,
from Jul. 1st, 2019, to Jun. 30th, 2020. 
\autoref{fig:preview}~(d) shows 
the original data.
The data are represented as the stream of the 3rd-order tensor, 
where each aspect indicates each \attribute.
Note that this tensor is sparse and high-dimensional, 
i.e., there are numerous dimensions in each aspect/attribute.
It does not exhibit 
any obvious patterns, 
neither regimes nor \topics.

\bit
\item \myiparapara{Regime identification}
As shown in \autoref{fig:preview}~(a),
\method
incrementally discovers 
nine regimes
(i.e., distinct time-evolving patterns).
Specifically, 
it finds Regimes~\#1 (blue) and
\#2 (orange),
corresponding to
weekdays and weekends, respectively.
Around the end of the year, 
it recognizes new regimes, 
Regimes~\#3, \#4 (green, red),
which coincided with certain festive days,
including Thanksgiving, Christmas, and Year-end.
The new Regimes~\#5 (purple) and \#6 (brown) indicate
abrupt changes in human movement.
In fact,
due to the emergence of a new viral pandemic, COVID-19,
the city ordered restaurants/bars to close on March 16th 
\cite{nyc_close_res}
and then 
offices to close on March 22nd
\cite{ny_pause_program}.
Finally,
our method generates Regimes~\#8 (gray) and \#9 (dark yellow)
for the new weekday and weekend human mobility patterns,
respectively,
after the reopening order on June 8th \cite{ny_reopen}, 
which allowed 
office-based workers and in-store retail shopping to resume.
\item \myiparapara{Multi-aspect \topic analysis}
\method provides
\topics
that are interpretable summaries for each \attribute.
\autoref{fig:preview}~(b) shows
the three major \topics for pick-up/drop-off locations in Regime~\#1,
where
we 
manually named them ``Entertainment'', ``Business'', and ``Night-out''.
These areas agree with our intuition:
the Entertainment \topic is allocated around Central Park and nearby museums,
the Business \topic is concentrated on major railway stations 
such as Grand Central Terminal,
and 
the Night-out \topic 
corresponds 
to the area around Lower Manhattan,
which has a large number of restaurants and bars.
\autoref{fig:preview}~(c) shows 
three major \topics for time attribute 
around mid-December.
They show the spiking of the Entertainment \topic 
during Christmas.
The Business \topic 
consistently exhibits high peaks on weekdays,
while it had lower value on weekends and 
Christmas.
Lastly, 
the Night-out \topic  
shows midnight peaks,
especially on weekends (i.e., Fri/Sat midnight).
\eit 

\mypara{Contributions}
The main contributions of our paper are:
\bit
    \item  
        \textbf{Effective}:
        We introduce dynamical multi-aspect patterns
        (i.e., 
        regimes and \topics),
        which summarize
        high-order 
        tensor streams
        and provide interpretable representations.
        Also, we formulate the summarization problem for capturing these patterns 
        in a data compression paradigm.
    \item
        \textbf{General}:
        To solve the summarization problem,
        we design 
        \method,
        which also performs
        data compression,
        pattern discovery,
        and anomaly detection.
        Our experimental results 
        show its practicality on multiple domains,
        such as online marketing analytics and cybersecurity.
    \item  
        \textbf{Scalable}:
        Our proposed algorithm is fast
        and requires constant computational time 
        both with regard to 
        the entire stream length 
        and the dimensionality
        for each attribute.
\eit
\myparaitemize{Reproducibility}
Our source code and datasets are available at~\cite{WEBSITE}.

\myparaitemize{Outline}
The rest of this paper is organized as follows. 
We first introduce related studies followed by our proposed model and algorithms,
experiments and conclusions.

%% file: 020related.tex
\TSK{
\input{TABLE/table_advantage} 
}
The mining of time-stamped event 
data
has attracted great interest in many fields
\cite{MatsubaraSF15,DBLP:conf/kdd/ShahKZGF15,takahashi2017autocyclone,yang2014finding,MatsubaraSF16,DBLP:conf/www/DeldariSXS21,DBLP:conf/kdd/XueZDDXZC20,DBLP:conf/www/MavroforakisVG17,DBLP:conf/www/YaoHZZA17,kamarthi2022camul,lu2022matrix,hu2021time}.
\autoref{table:capability} illustrates the relative advantages of our method, and only \method meets all the requirements.

\myparaitemize{Modeling Dynamics and Segmentation}
Classical approaches 
such as 
linear dynamical systems (LDS), 
and hidden Markov models (HMM) 
are extended to capture distinct patterns of sequences
as described 
in
\cite{li2010parsimonious,HooiLSF17,tozzo2021statistical,kawabata2019automatic}.
TICC \cite{hallac2017toeplitz}
characterizes the interdependence between multivariate observations 
based on a Markov random field.
Such distinct time series patterns also enable us 
to perform anomaly detection and forecasting \cite{MatsubaraSF14,chen2018neucast,MatsubaraS16,MatsubaraS19}.
Tensor-based approaches 
for time series segmentation 
have been proposed \cite{honda2019multi,kawabata2020non}
that incorporate latent relationships between sequences.
The previous studies are designed for continuous time series
and are thus incapable of modeling sparse tensors.
In recent years, 
many deep neural network models have been proposed
\cite{ma2019learning,lee2020temporal}.
T-LSTM \cite{DBLP:conf/kdd/BaytasXZWJZ17}
identifies disease progression patterns with 
irregular time intervals.
Since these are mostly ``black-box'' models and 
incur high computation costs,
they cannot address streaming data summarization.
Moreover, 
none of the above studies focuses on 
large and sparse tensors with 
a higher order than $3$.

\myparaitemize{Summarization and Clustering}
Probabilistic generative models
\cite{beutel2014cobafi,DBLP:conf/kdd/OkawaIK0TU19}, 
such as latent Dirichlet allocation (LDA) \cite{blei2003latent}
and its variants \cite{iwata2009topic,DBLP:conf/kdd/MengZH0Z020,yin2018model},
are broadly applied to 
analyze large collections of categorical data.
More recently, 
topic models have been extended to 
neural-based models 
\cite{WangLCKLS19,3442381.3449943} 
by using a variational autoencoder
\cite{kingma2013auto}.
A collection of events 
can be turned into a tensor
\cite{schein2015bayesian,hooi2019smf,jang2021fast}.
TriMine \cite{MatsubaraSFIY12} 
summarizes an event tensor and discovers groups of dimensions.
As with 
\cite{araujo2014com2,DBLP:conf/kdd/ShahKZGF15,koutra2014vog,MatsubaraSF16},
the minimum description length (MDL) principle \cite{grunwald2005advances} 
is applied to summarize time series and dynamic graphs.
Unlike these methods,
our work focuses on 
tensor streams.
The processing and clustering of data streams
\cite{10.5555/1315451.1315460,gong2017clustering,hahsler2016clustering,mansalis2018evaluation},
such as DBSTREAM \cite{hahsler2016clustering},
have also attracted significant interest.
However, 
these algorithms process each data point individually 
and cannot capture multi-aspect features.

\myparaitemize{Anomaly Detection}
Typical anomaly detection methods 
\cite{smets2011odd,10.1145/1541880.1541882,FANAEET2016130,akoglu2015graph,manzoor2018xstream}, 
such as local outlier factor (LOF) \cite{10.1145/342009.335388} 
and tree-based approaches \cite{liu2012isolation,10.5555/3045390.3045676}, 
can be used in event tensors by converting multiple attributes to numerical ones.
\cite{jiang2015general,mao2014malspot,shi2015stensr,shin2017densealert,bhatia2021mstream}
use a stream of multi-aspect records as input.
MemStream \cite{DBLP:conf/www/0001JSKH22} 
can learn dynamically changing trends
to handle time-varying data distribution known as {\it concept drift} \cite{lu2018learning,gupta2013outlier,chi2017hashing}.
Although these methods have the ability to detect multiple anomalies,
they cannot identify the types of anomalies
or capture dynamical multi-aspect patterns.

In conclusion, 
none of the existing methods focus specifically 
on modeling of
dynamical multi-aspect patterns,
summarization, and anomaly detection in high-order tensor streams. 

%% file: TABLE/table_advantage.tex
\newcommand*\rot{\rotatebox{90}}

\newcommand*\OK{\ding{51}}
\newcommand*\NO{-}
\newcommand{\SOME}{some}

\begin{table}[t]
\centering
\vspace{-0.8em}
\caption{
Capabilities of approaches.
}
\label{table:capability}
\vspace{-1.2em}
\scalebox{0.8}{
\begin{tabular}{l|ccc|ccc|cc|c}
\toprule
&
\rot{TICC/++} &
\rot{CubeMarker} &
\rot{T-LSTM} & 
\rot{LDA/NTM/++} &
\rot{\textsc{Trimine}} &
\rot{DBSTREAM/++} &
\rot{LOF/++}&
\rot{MemStream } &
\rot{\method} \\

\midrule
\rowcolor{lightgray}
High-dimensional Tensor  &\NO &\SOME &\NO &\NO &\OK &\NO &\NO &\OK &\OK\\
Sparsity                    &\NO &\NO &\OK &\OK &\OK &\NO &\NO &\OK &\OK\\
\rowcolor{lightgray} 
Semi-infinite Data          &\NO &\NO &\NO &\NO &\NO &\OK &\NO &\OK &\OK\\
%
\midrule
Segmentation  &\OK &\OK &\OK &\NO &\NO &\NO &\NO &\NO &\OK\\
\rowcolor{lightgray} 
Data Compression            &\OK &\OK &\NO &\OK &\OK &\OK &\NO &\NO &\OK\\
Anomaly Detection  &\NO &\NO &\NO &\NO &\NO &\NO &\OK &\OK &\OK\\
\rowcolor{lightgray} 
Dynamical Multi-aspect Patterns &\NO &\NO &\NO &\NO &\NO &\NO &\NO &\NO &\OK\\
\bottomrule
\end{tabular}
}
\normalsize
\end{table}

%% file: 030model.tex
In this section, we present our proposed model.
\subsection{Design Philosophy of \method}
\label{sec:designphilo}
The symbols used in this paper 
are described
in Appendix A.
Here we consider 
our settings, namely,
complex time-stamped event streams.
We continuously monitor
an event entry with 
$\nmode$ categorical \attributes and a timestamp.
At the most recent time $\duration$,
we have a collection of events with
$\nunits_1 \ldots \nunits_{\nmode}$ unique \units 
for each \attribute and 
$\duration$ timestamps.
\begin{definition}[\Etstream]
\rm{
Let $\tensor \in \mathN^{\nunits_1 \times \cdots \times \nunits_{\nmode} \times \duration}$ be an $\nmode+1$ th-order tensor stream up to the current 
time point $\duration$.
At every time point $\duration$ 
that is arrived at with a non-overlapping time interval $\tau \ll \duration$,
we can obtain the current tensor $\tensorC \in \mathN^{\nunits_1 \times \cdots \times \nunits_{\nmode} \times \tau}$ 
as the partial tensor of $\tensor$.
The element $x_{\lunit_1 \ldots \lunit_{\nmode} ,\ltime}$ of $\tensor$ shows the total number of 
event entries 
of the $\lunit_1$-th \ldots $\lunit_{\nmode}$-th \units 
in each \attribute 
at time tick $\ltime$.
}
\end{definition}
\autoref{fig:preview}~(d) shows the event tensor stream for NYC taxi rides,
where each event is of the form
{\it (pick-up ID, drop-off ID, \timestamp)},
$\nmode=2$.
Here, we provide the reader with three important observations.

\begin{observation}[High Dimensional]
This tensor has a large number of \units in each attribute,
e.g., $\nunits_1=262$ and $\nunits_2=264$ 
in the pick-up/drop-off location attribute.
\end{observation}
\begin{observation}[Sparse]
In the figure,
most of the attribute 
pairs 
have
very sparse sequences,
which derails all typical time series analysis tools because they look like noise (e.g., $\{0,0,0,1,0,$ $0,1,2,0,\cdots\}$).
\end{observation}
\begin{observation}[Semi-infinite]
The \etstream evolves over time and arrives 
in an unbounded stream,
making it impossible to store all the historical data.
\end{observation}

Consequently,
we aim to summarize the \etstream and 
obtain a succinct description.
Specifically,
we focus on the 
two types of patterns, 
\prop{P1} \topics (i.e., latent groups and their relationship)
and \prop{P2} regimes (i.e., distinct time-evolving patterns).
So, what is the simplest mathematical model 
that can capture both \prop{P1} and \prop{P2}? 
How can we formulate the summarization problem?
We provide the answers below.

\subsection{Proposed Solution: \method}
\label{sec:model}
We now present our model in detail.
We first describe
\prop{P1} \topics
in each current tensor $\tensorC$
by introducing
{\it multi-aspect \topic factorization}
and then propose a {\it compact description}
for representing
\prop{P2} regimes and the whole tensor stream.
Finally, 
we formalize the problem
as minimizing {\it encoding cost} in the data compression paradigm.

\subsubsection{Multi-aspect \Topic Factorization (P1)}
\label{section:decomp}
We begin with the simplest case,
where we have only a current tensor $\tensorC$.
Our first step is to describe 
a high-dimensional and sparse tensor $\tensorC$ 
as a compact and interpretable model.
We thus propose a new factorization
to model the generative process of events.
In our model, 
we assume that there are $\ntopic$ major trends/\topics
behind the event collections.
Specifically,
the $\ltopic$-th \topic is characterized
by probability distributions 
in terms of $\nmode$ attributes and time,
which are defined as follows:
\bit
\item
$ \Matt^{(\lmode)}_{\ltopic} \in \mathR^{\nunits_{\lmode}}$:
probability distribution 
over $\nunits_{\lmode}$ \units of the attribute $\lmode$ for
the \topic $\ltopic$.

\item
$ \Mtime_{\ltime} \in \mathR^{\ntopic}$:
probability distribution over $\ntopic$ \topics
for the time $\ltime$.
\eit
Here, 
we refer to
$\Matt^{(1)}, \ldots, \Matt^{(\nmode)}$, and $\Mtime$
as \topic matrices.
Since we treat each attribute as categorical,
the \topic matrices can be described 
by employing a Dirichlet prior \cite{blei2003latent}:
\small
\begin{align}
\Matt_{\ltopic}^{(\lmode)} \sim \textrm{Dirichlet}(\alpha^{(\lmode)}),
    ~
    ~
\Mtime_{\ltime} \sim \textrm{Dirichlet}(\beta),
    \nonumber
\end{align}
\normalsize
where $\alpha^{(m)}$ and $\beta$ are 
the hyperparameters\footnote{
We set $\alpha^{(m)} = \beta = 1/\ntopic$ as default.
}.

We also incorporate temporal dependencies 
into this model so that each \topic matrix 
captures
the context of its predecessors in the data stream.
We assume that
the means of the \topics
are the same as at the previous time $\duration-\tau$,
unless the newly arrived events 
$\tensorC 
$
are confirmed.
With this assumption,
we can use the following Dirichlet priors:
$\textrm{Dirichlet}(\alpha^{(\lmode)}{}_{1}\Patt_{\ltopic}^{(\lmode)})$
and 
$\textrm{Dirichlet}(\beta{}_{1}\Ptime_{\ltime})$, 
where $_{l}\Patt^{(\lmode)}$ and $_{l}\Ptime$ are the previous \topic matrices
at $\duration-l\tau$.
%
To capture the long-term dependencies, 
we can extend this approach so that it can depend on past $L$ matrices.
\small
\begin{align}
\Matt_{\ltopic}^{(\lmode)} &\sim \textrm{Dirichlet}(\Sigma_{l=1}^{L} \alpha^{(\lmode)}{}_{l}\Patt^{(\lmode)}_{\ltopic}),
    ~~
\Mtime_{\ltime} \sim \textrm{Dirichlet}(\Sigma_{l=1}^{L} \beta{}_{l}\Ptime_{\ltime}).
    \nonumber
\end{align}
\normalsize
Consequently, the generative process can be described 
as follows:
\begin{center}
  \fbox{
  \hspace{-2em}
  \begin{minipage}{1\columnwidth}
\small
\begin{itemize}
\setlength{\parskip}{0cm}\setlength{\itemsep}{0.1cm}
\item 
  For each \topic $\ltopic=1, \ldots, \ntopic$: 
  \begin{itemize}
  \setlength{\parskip}{0cm}\setlength{\itemsep}{0.1cm}
  \renewcommand{\labelenumi}{(\alph{enumi})}
    \item
    For each attribute $\lmode=1, \ldots, \nmode$: 
    \begin{itemize}
    \setlength{\parskip}{0cm}\setlength{\itemsep}{0.1cm}
    \renewcommand{\labelenumi}{(\alph{enumi})}
      \item 
       $ \Matt^{(\lmode)}_{\ltopic} \sim \textrm{Dirichlet}(\Sigma_{l=1}^{L} \alpha^{(\lmode)}
       {}_{l}\Patt^{(\lmode)}_{\ltopic})$
    \end{itemize}
  \end{itemize}
\item
  For each time $\ltime=1, \ldots, \tau$:
  \begin{itemize}
  \setlength{\parskip}{0cm}\setlength{\itemsep}{0.1cm}
  \renewcommand{\labelenumi}{(\alph{enumi})}
    \item
      $\Mtime_{\ltime} 
      \sim 
      \textrm{Dirichlet}(\Sigma_{l=1}^{L} \beta {}_{l}\Ptime_{\ltime})$
    \item
      For each entry $j=1, \ldots, N_{\ltime}$: 
      \begin{itemize}
      \setlength{\parskip}{0cm}\setlength{\itemsep}{0.1cm}
      \renewcommand{\labelenumi}{(\alph{enumi})}
        \item
        $z_{\ltime,j} \sim \textrm{Multinomial}(\Mtime_{\ltime})$
        // Draw a latent \topic $z_{\ltime,j}$ 
        \item 
        For each \attribute $\lmode=1, \ldots, \nmode$: 
        \begin{itemize}
        \setlength{\parskip}{0cm}\setlength{\itemsep}{0.1cm}
        \renewcommand{\labelenumi}{(\alph{enumi})}
          \item
          $e^{(\lmode)}_{\ltime,j} \sim \textrm{Multinomial}(\Matt^{(\lmode)}_{z_{\ltime,j}})$,
          // Draw a \unit in each \attribute
        \end{itemize}
      \end{itemize}
    \end{itemize}
\end{itemize}
\normalsize
\end{minipage}
}
\end{center}
where 
$N_{\ltime}$ is the total number of events at time $t$, 
and $z_{\ltime,j}$ is the latent \topic .
Each event $e_{\ltime,j}$ is sampled from the \topic -specific multinomials.
We note that the benefits of this model are three-fold.
{\it First}, even though the event tensor has sparse activity,
our model can discard many redundancies (e.g., noise) and 
summarize a set of events into $\ntopic$ \topics.
{\it Second}, 
an event entry 
is generated from
$\nmode+1$ \topic matrices.
It thus handles arbitrary-order tensors.
{\it Third},
to capture temporal dependencies, 
it employs past $L$ \topic matrices rather than storing tensors.

\subsubsection{Compact Description (P2)}
Although the \topic matrices 
concisely describe
the partial tensor $\tensorC$,
it is insufficient 
for 
the whole tensor stream $\tensor$,
containing various types of distinct
dynamical patterns.
We thus introduce another higher-level architecture.

\begin{definition}[Regime: $\regime$]
\rm{
Let $\regime$ be a regime consisting of the \topic matrices:
$\regime = \{ \{\Matt^{(\lmode)}\}_{\lmode=1}^{\nmode}, \Mtime \}$
to describe a certain distinct dynamical pattern
with which we can divide and summarize the entire tensor stream into segments.
When there are $\nregime$ regimes, 
a regime set is defined as $\regimeset =\{\regime_\lregime\}_{\lregime=1}^{\nregime}$.
}
\end{definition}

Also, when there are $\nshiftp$ switching positions, 
the regime assignments are defined as $\regimeassignment =\{\regimeassign_\lshiftp\}_{\lshiftp=1}^{\nshiftp}$,
where $\regimeassign_\lshiftp =(t_s,\lregime)$ is the history of each switching position $t_s$ to the $\lregime$-th regime.
Finally, 
we adopt all the above parts for 
a compact description of $\tensor$.
\begin{definition}[Compact description]
\rm{
Let $\cand = \{\nregime, \Theta, \nshiftp, \regimeassignment \}$ 
be 
a compact representation of the whole tensor stream $\tensor$,
    namely,
    \bit  
    \item
    the number of regimes $\nregime$ and the regime set, 
    $ \regimeset =\{\regime_\lregime\}_{\lregime=1}^{\nregime}$, 
    \item
    the number of segments $\nshiftp$ and the assignments, 
    $ \regimeassignment =\{\regimeassign_\lshiftp\}_{\lshiftp=1}^{\nshiftp}$.
    \eit
}
\end{definition}




\subsubsection{Problem Formulation}
Our final goal is to formulate the problem, 
where we summarize all data streams $\tensor$  
into a compact representation $\cand$.
Our objective function leverages 
the minimum description length (MDL) principle \cite{grunwald2005advances}.
In short, 
it follows the assumption that
the more we can compress the data, the more we can learn about their underlying patterns.
Specifically,
we evaluate the total encoding cost,
which can be used to losslessly compress the original tensor stream $\tensor$.
The summarization problem is written as follows:
\begin{problem} 
  \label{comp_problem}
    \textbf{Given}
    a whole event stream $\tensor$,
    \textbf{find}
    the compact description $\cand$, which minimizes the total encoding cost
    \small
    \begin{align}
    \label{eqn:cost:total}
    \costT{\tensor}{\cand} &= \costM{\cand} + \costC{\tensor}{\cand},
    \end{align}
    \normalsize
    where  $\costM{\cand}$ is the \textit{model coding cost} of $\cand$,
    and $\costC{\tensor}{\cand}$ is 
    the \textit{data coding cost} given the model $\cand$.
\end{problem}
\hide{
Given a whole event stream $\tensor$ and compact description $\cand$,
the total encoding cost is described as follows:
\small
\begin{align}
    \label{eqn:cost:total}
    \costT{\tensor}{\cand} &= \costM{\cand} + \costC{\tensor}{\cand},
\end{align}
\normalsize
where  $\costM{\cand}$ is the \textit{model coding cost} of $\cand$,
and $\costC{\tensor}{\cand}$is the \textit{data coding cost} given the model $\cand$.
In the following, we define the two parts of the total cost more concretely.
\\
}
\myparaitemize{Model Coding Cost}
The model coding cost is the number of bits needed to describe the model.
In our model,
the dimensionality of latent components requires 
$\costM{d}=\Sigma_{\lmode=1}^\nmode \log^{*}(\nunits_{\lmode}) 
+ \log^{*}(\tau) + \log^{*}(\ntopic)$
\footnote{Here, log$^*$ is the universal code length for integers.}.
The number of regimes needs $\costM{\nregime} = \log^{*}(\nregime)$.
The model coding cost of each regime $\regime$ consists of the following terms,
\small
\begin{align}
    \label{eqn:cost:regime}
    &\costM{\regime} = 
    \sum
    _{\lmode=1}^{\nmode} \costM{\Matt^{(\lmode)}} + \costM{\Mtime} , \\
    &\costM{\Matt^{(\lmode)}} 
    = |\Matt^{(\lmode)}| \cdot(\log(\ntopic) + \log(\nunits_{\lmode}-1) + \cF) + \log^*(|\Matt^{(\lmode)}|), \\
    &\costM{\Mtime} = |\Mtime| \cdot (\log(\tau) + \log(\ntopic-1) + \cF) + \log^*(|\Mtime|), 
\end{align}
\normalsize
where $|\cdot|$ describes the number of non-zero elements in each of the matrices,
and $\cF$ is the 
floating point
cost\footnote{
We set 
$8$
bits as the default by following \cite{MatsubaraSF15,takahashi2017autocyclone}.}.
The number of segments needs $\costM{\nshiftp} = \log^{*}(\nshiftp)$.
Each shifting point needs 
$\costM{\regimeassign_\lshiftp}=\log^{*}(t_s)+\log(\nregime)$.

\myparaitemize{Data Coding Cost}
Given a full regime set $\regimeset$, 
we can encode the data $\tensor$ based on Huffman coding \cite{bohm2007ric},
i.e, a number of bits are assigned to each value in $\tensor$. 
The data coding cost of $\tensor$ given $\theta$ is computed by:
$\costC{\tensor}{\theta}= -\log P(\tensor|\{\Matt^{(m)}\}_{\lmode=1}^{\nmode},\Mtime)$.
Thus, the data coding cost of $\tensor$ given $\cand$ is computed by:
\small
\begin{align}
    \label{eqn:cost:code}
    \costC{\tensor}{\cand}
        &= \sum_{\lregime=1}^{\nregime}  - \log P(\tensor[\lregime]|\theta_\lregime),
\end{align}
\normalsize
where
$\tensor[\lregime]$ is a set of partial tensors assigned by the $\lregime$-th regime.
Finally,
the total encoding cost $\costT{\tensor}{\cand}$ is written as follows:
\begin{center}
\fbox{
\begin{minipage}{0.95\columnwidth}
\small
\begin{align}
    \label{eqn:cost:total_cost}
    \costT{\tensor}{\cand}
    \nonumber
    &= \costM{\cand} + \costC{\tensor}{\cand}\\
    \nonumber
    &= \costM{d}+\costM{\nregime} +\costM{\nshiftp}\\
    &+ \sum
    _{\lregime=1}^{\nregime}\costM{\regime}
    + \sum
    _{\lshiftp=1}^{\nshiftp}\costM{\regimeassign_\lshiftp}
    + \costC{\tensor}{\cand}.
  \end{align}    
\normalsize
\end{minipage}
}
\end{center}

%% file: 040proposed.tex
Thus far, 
we have described 
how 
we 
represent the two concepts,
(i.e., \topics and regimes),
and formulate the summarization problem 
(i.e., Problem~\ref{comp_problem})
in a lossless compression context.
Our next goal is to solve the problem
in a streaming setting.
In this section,
we aim to figure out {\it how to incrementally summarize}
entire event streams into a compact description $\cand$ 
and also 
{\it how to exploit} the compact description 
for streaming 
anomaly detection.

To tackle these problems, 
we now present a streaming algorithm \methodalgo,
consisting of two sub-algorithms, 
\decomposition and \compression .
Algorithm ~\ref{alg:main} 
(see Appendix~\ref{sec:app:algo})
shows the overall procedure, and 
Figure~\ref{fig:overview} illustrates
how the proposed algorithm works.
Intuitively, 
our algorithm continuously generates a regime $\regime_c$ from 
the non-overlapping arrival tensor $\tensorC$.
It then updates the compact description $\cand$ 
with $\regime_c$
and 
measures
the anomalousness of $\tensorC$.
Next, we describe each sub-algorithm in detail.

\subsection{\decomposition}
We first aim to incrementally monitor $\tensorC$ and 
estimate a candidate regime $\regime_c$ 
(i.e., $\{\Matt^{(\lmode)}\}_{\lmode=1}^{\nmode}$, $\Mtime$),
which best describes $\tensorC$.
According to
the generative process in section~\ref{section:decomp},
we 
propose an efficient estimation 
with 
collapsed Gibbs sampling \cite{porteous2008fast}.
Specifically,
for each non-zero entry 
$\Ex_{\lunit_1,\ldots,\lunit_{\nmode},\ltime}$ 
in $\tensorC$,
we draw latent \topics 
$z_{\lunit_1,\ldots,\lunit_{\nmode},\ltime}$ 
with the probability $p$:
\small
\begin{align}
\label{eqn:gibbs_sampling_time}
  &p(z_{\lunit_1,\ldots,\lunit_{\nmode},\ltime} = 
\nonumber  
\ltopic ~|~ \tensorC, \Mtime', \Ptime, \beta, \{\Matt^{(\lmode)'}, \Patt^{(m)} , \alpha^{(\lmode)}\}_{\lmode=1}^{\nmode} ) \\ 
    &\propto
    \frac{\Etime'_{\ltime,\ltopic}
    +\sum_{l=1}^L\beta{}_{l}\ptime_{\ltime,\ltopic}}
    {\sum_{\ltopic=1}^{\ntopic} \Etime'_{\ltime,\ltopic} 
    + L\beta }
    \cdot
    \prod_{\lmode=1}^{\nmode}
    \frac{\Eatt^{(\lmode)'}_{\ltopic,\lunit_\lmode}
    +\sum_{l=1}^L\alpha^{(\lmode)}{}_{l}\patt^{(\lmode)}_{\ltopic,\lunit_\lmode}}
    {\sum_{\lunit=1}^{\nunits_\lmode}\Eatt^{(\lmode)'}_{\ltopic,\lunit_\lmode} 
    + L\alpha^{(\lmode)}},
\end{align}
\normalsize
where $\Eatt_{\ltopic,\lunit_{\lmode}}^{(\lmode)}$ and $\Etime_{\ltime,\ltopic}$ are the total counts
of that \topic $\ltopic$ is assigned to the $\lunit_{\lmode}$-th \unit
and time-tick $\ltime$, respectively.
Note that the prime 
(e.g., $\Etime'_{\ltime,\ltopic}$) 
indicates the
count
yielded by excluding the entry
$\Ex_{\lunit_1,\ldots,\lunit_{\nmode},\ltime}$.
After the sampler has burned-in, we can obtain
the component matrices for $\tensorC$, i.e.,
$\{\tilde{\Matt}^{(\lmode)}\}_{\lmode=1}^{\nmode}$ and $\tilde{\Mtime}$, as follows: 
\small
\begin{align}
    \label{eqn:normalizedMatrices} 
    \tilde{\Matt}_{\ltopic,\lunit_\lmode}^{(\lmode)} \propto
    \frac{\Eatt_{\ltopic,\lunit_\lmode}^{(\lmode)}
    +\sum_{l=1}^{L}\alpha^{(\lmode)}{}_{l}\patt_{\ltopic,\lunit_\lmode}^{(\lmode)}}
    {\sum_{\lunit=1}^{\nunits_{\lmode}} \Eatt_{\ltopic,\lunit_\lmode}^{(\lmode)} 
    + L\alpha^{(\lmode)}},
    ~
    \tilde{\Mtime}_{\ltime,\ltopic} \propto
    \frac{\Etime_{\ltime,\ltopic}+\sum_{l=1}^{L}\beta{}_{l}\ptime_{\ltime,\ltopic}}{\sum_{\ltopic=1}^{\ntopic} \Etime_{\ltime,\ltopic} + L\beta}.
\end{align}
\normalsize

Algorithm~\ref{alg:decomp} 
(see Appendix~\ref{sec:app:algo})
shows \decomposition in detail.
It first assigns the latent \topic
$z_{\lunit_1,\ldots,\lunit_{\nmode},\ltime}$ 
for each entry
$\Ex_{\lunit_1,\ldots,\lunit_{\nmode},\ltime}$
in $\tensorC$ 
with Equation ~(\ref{eqn:gibbs_sampling_time}).
Once the latent \topics are determined, 
we can compute the objective matrices simply with Equation ~(\ref{eqn:normalizedMatrices}).
Here, we keep past $L$ \topic matrices in the past parameter set,
i.e., a FIFO queue $Q$ with size $L$.
After computing the current \topic matrices,
the oldest matrices are removed,
and 
the estimated matrices are inserted into $Q$.

Even though 
our tensor data is
high-dimensional
(Observation~$1$),
\decomposition
does not depend on 
dimensionality,
i.e., it takes $O(N)$ time, 
where $N$ is the number of events (see Lemma~\ref{lemma:time} for details).
In contrast,
conventional tensor 
algorithms such as
alternating least squares
(ALS)
\cite{kolda2009tensor}
scale with respect to all the attributes,
i.e., take $O(\prod_{\lmode=1}^{\nmode}\nunits_\lmode)$ time,
which
may become very computationally demanding 
for high-order tensors ($M \geq 4$).

\TSK{
\begin{figure}[t]
    \centering
    \includegraphics[width=\columnwidth]{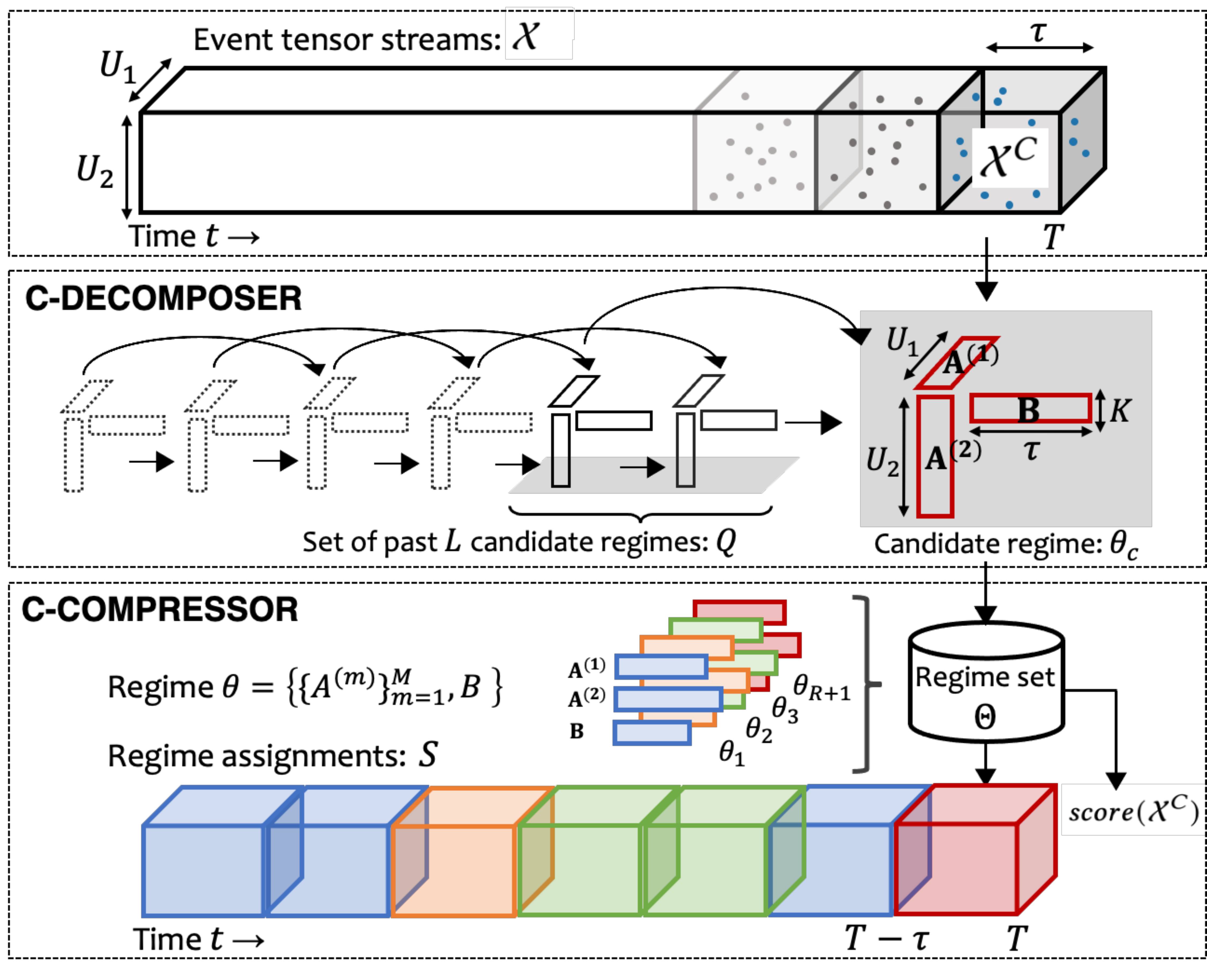}
    \vspace{-2.8em}
    \caption{
        Illustration of \method algorithm.
        \decomposition:
        Given a current tensor $\tensorC \in \mathN^{\nunits_1 \times \nunits_2 \times \tau}$,
        it first decomposes $\tensorC$ into a candidate regime $\regime_c$ 
        depending on past $L$ regimes ($L=2$ in this figure).
        \compression: it assigns the optimal regime for $\tensorC$ 
        among the candidate regime $\regime_c$ and the regimes $\{\regime_\lregime\}_{\lregime=1}^{\nregime}(= \regimeset)$.
        If the candidate regime is assigned,
        \method inserts it into the regime set $\regimeset$ as $\regime_{\nregime+1}$.
        It also computes the anomalousness score for $\tensorC$.
    }
    \label{fig:overview}
    \vspace{0.5em}
\end{figure}
}

\subsection{\compression}
We next describe \compression in steps.
Algorithm~\ref{alg:compression} 
(see Appendix~\ref{sec:app:algo})
shows the overall procedure.
After obtaining the candidate regime $\regime_c$,
our next goal is 
to find compact description $\cand$ 
for the whole tensor stream $\tensor$.
However, $\tensor$ is semi-infinite (Observation~3),
and we thus
cannot process all historical data.
To efficiently find the compact description $\cand$,
we adopt an insertion-based algorithm,
where 
it maintains a reasonable description for $\tensor$ 
and generates a regime if necessary.
Specifically,
the algorithm tracks only two regimes,
the previous regime $\regime_p$ and the candidate regime $\regime_c$.
Given a current tensor $\tensorC$ with the two regimes,
the algorithm compares the extra cost
according to Equation~(\ref{eqn:cost:total_cost})
for each regime,
and then chooses the next procedure so that the additional cost is minimized:
\begin{itemize}
    \item 
    If \compression uses $\regime_p$,
    it stays in the previous regime.
    
    \item
    If $\regime_c$ is chosen,
    \compression
    first finds a more suitable regime in $\regimeset$   
    to avoid duplication.
    Then, the least expensive regime is adopted.
\end{itemize}
Here, the additional cost $\costT{\tensorC}{\regime_{*}}$ is written as follows:
\begin{center}
\fbox{
\begin{minipage}{0.95\columnwidth}
\small
\begin{align}
    \costT{\tensorC}{\regime_{*}} 
    &= \Delta\costM{\cand} + \costC{\tensorC}{\regime_{*}},\\
    \label{eqn:cost:delta_cost_model}
    \nonumber
    \Delta \costM{\cand}
    \nonumber
    &= \log^{*}(\nregime+1) - \log^{*}(\nregime)~ + \costM{\regime_*} \\
    &+ \log^{*}(\nshiftp+1) - \log^{*}(\nshiftp)~ + \costM{\regimeassign},
  \end{align}    
\normalsize
\end{minipage}
}
\end{center}
where $\regime_*$ indicates any regime.
If 
we need 
to shift another existing regime to represent $\tensorC$,
then $\Delta\costM{\cand} = \log^{*}(\nshiftp+1) - \log^{*}(\nshiftp) + \costM{\regimeassign}$;
if the description of $\tensorC$ requires new regimes, 
it costs all of the terms in \autoref{eqn:cost:delta_cost_model};
otherwise, $\Delta\costM{\cand} = 0$.

\myparaitemize{Online Regime Updates}
Whenever an existing regime is selected,
each count in the existing regime is updated by adding $\regime_c$:
\small
\begin{align}
    \label{eqn:model_update}
    \tilde{{\Matt}}_{\ltopic,\lunit_\lmode}^{(\lmode)}&
    \leftarrow
     \frac
     {\Eatt_{\ltopic,\lunit_\lmode}^{(\lmode)}
     +\sum_{l=1}^{L}\alpha^{(\lmode)}{}_{l}\patt_{\ltopic,\lunit_\lmode}^{(\lmode)} 
     + {}_{c}\Eatt_{\ltopic,\lunit_\lmode}^{(\lmode)}}
     {\sum_{\lunit=1}^{\nunits_{\lmode}} \Eatt_{\ltopic,\lunit_\lmode}^{(\lmode)}
     +L\alpha^{(\lmode)}
     +\sum_{\lunit=1}^{\nunits_{\lmode}}{}_{c}\Eatt_{\ltopic,\lunit_\lmode}^{(\lmode)}},
\end{align}
\normalsize
where $c$ (e.g., $_{c}\Eatt_{\ltopic,\lunit_\lmode}^{(\lmode)}$)
is a count in the candidate regime.
$\tilde{{\Mtime}}_{\ltime,\ltopic}$ is analogous, 
and omitted for brevity. 
The effect of the candidate regime is decayed as the existing regime is updated;
in other words, each regime converges as it updates.

\myparaitemize{Anomaly Detection}
Finally, we exploit the compact description $\cand$ for anomaly detection.
Compression-based techniques are 
naturally suited for anomaly and rare instance detection.
In a given compact description $\cand$,
the high usage regime 
compresses the majority of all past data 
with a short code length.
In other words, 
it represents the {\it norm} in the data stream
and thus it needs long code length against rare instances.
We thus consider the encoding cost of $\tensorC$ 
as its anomalousness score;
the higher the compression cost, the more likely it is 
``to arouse suspicion that 
it was generated by a different mechanism''
\cite{hawkins1980identification}.
\small
\begin{align}
    \nonumber
    &\norm =~ \argmax_{\lregime \in
    \nregime}|\segmentset^{-1}_{\lregime}|,\\
    &\score{\tensorC} =~ \costC{\tensorC}{\regime_{\norm}},
    \label{eqn:anomaly_score}
\end{align}
\normalsize
where $|\segmentset^{-1}_{\lregime}|$ is
the total segment length of the regime $\regime_\lregime$.
Note that, in data streams, 
the concept of normal changes over time,
and this is 
known as {\it concept drift} \cite{DBLP:conf/www/0001JSKH22}.
This approach
can adaptively change the 
norm
to judge incoming tensors
as the concept drift.

\begin{lemma}[Time complexity of \method]
    The time complexity of \methodalgo is at least $O(\ntotal)$ and at most $O(\ntotal + \nregime)$ per process,
    where
    $\nregime$ is the number of regimes and  
    $\ntotal$ is the total number of event entries in $\tensorC$ 
    (i.e., $\sum_{\lunit_1} \cdots \sum_{\lunit_{\nmode}} \sum_{\ltime} \Ex_{\lunit_{1},\ldots,\lunit_{\nmode},\ltime}$).
    \label{lemma:time}
    \end{lemma}
    \begin{proof}
    Please see Appendix~\ref{sec:app:algo}.
    \end{proof}
\vspace{-1em}

%% file: 050experiments.tex
\TSK{
\input{TABLE/table_datasets}
}

\TSK{
\begin{figure*}[!t]
        \vspace{-3em}
    \begin{tabular}{cc}
        \begin{minipage}{0.65\columnwidth}
            \vspace{2em}
            \centering
            \includegraphics[width=\columnwidth]{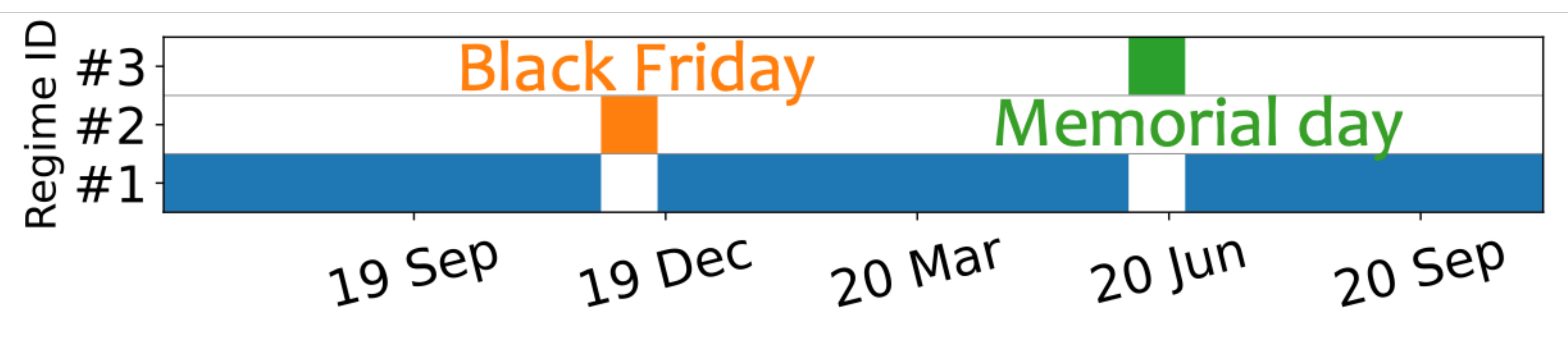}\\
            \vspace{-0.8em}
            (a) Regime identification \\
            \centering
            \includegraphics[width=\columnwidth]{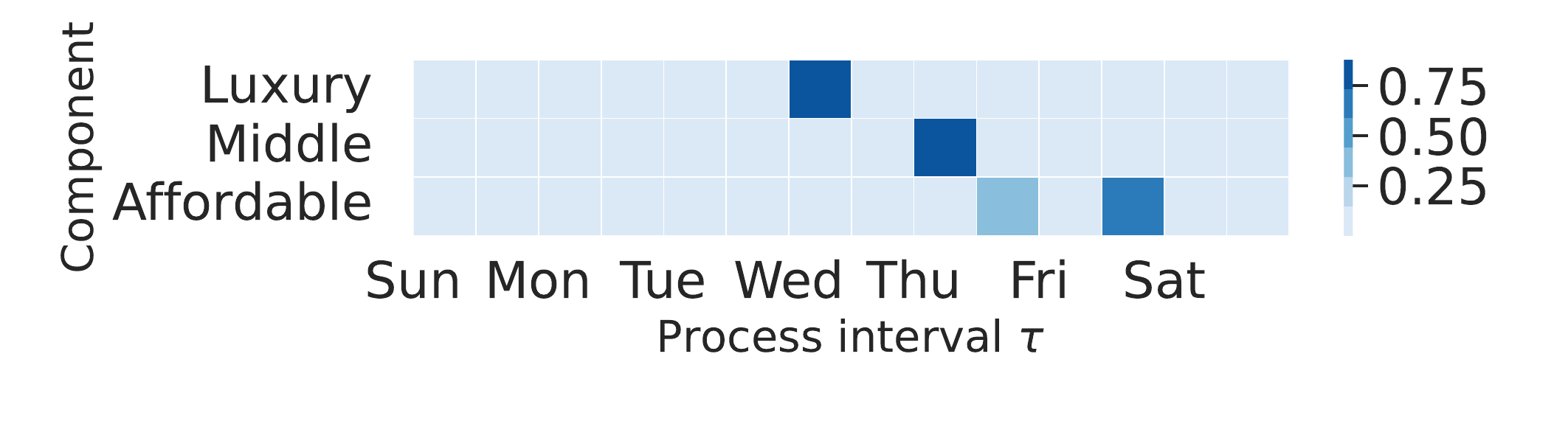}\\
            \vspace{-1em}
            (b-i) Regime\#1 (normal sale) \\ 
            \centering
            \includegraphics[width=\columnwidth]{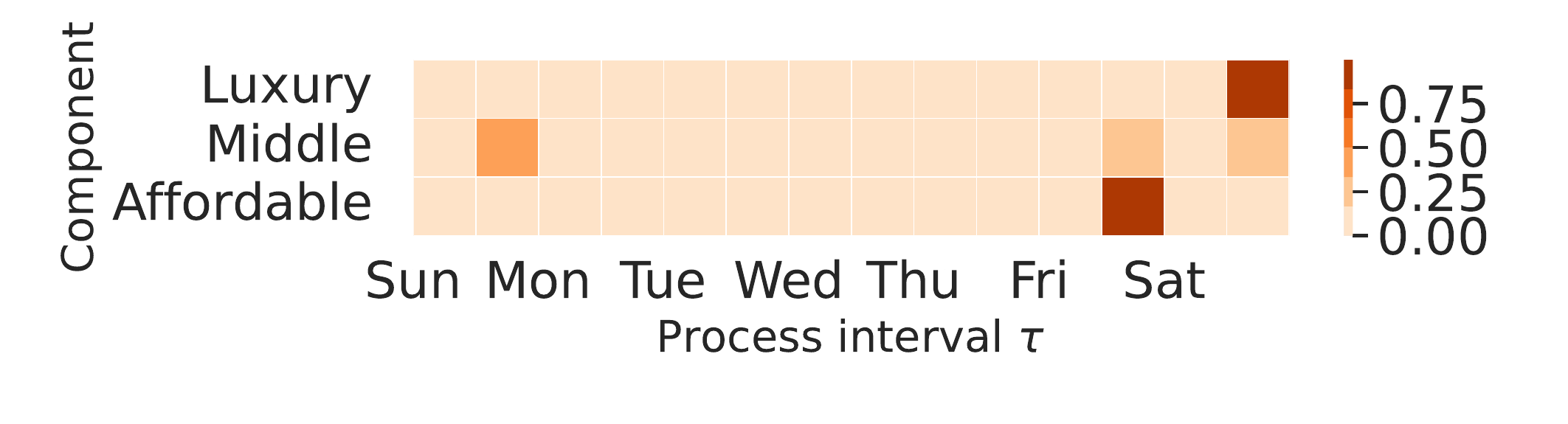}\\
            \vspace{-1em}
            (b-ii) Regime\#2 (Black Friday sale) \\
            \centering
            \includegraphics[width=\columnwidth]{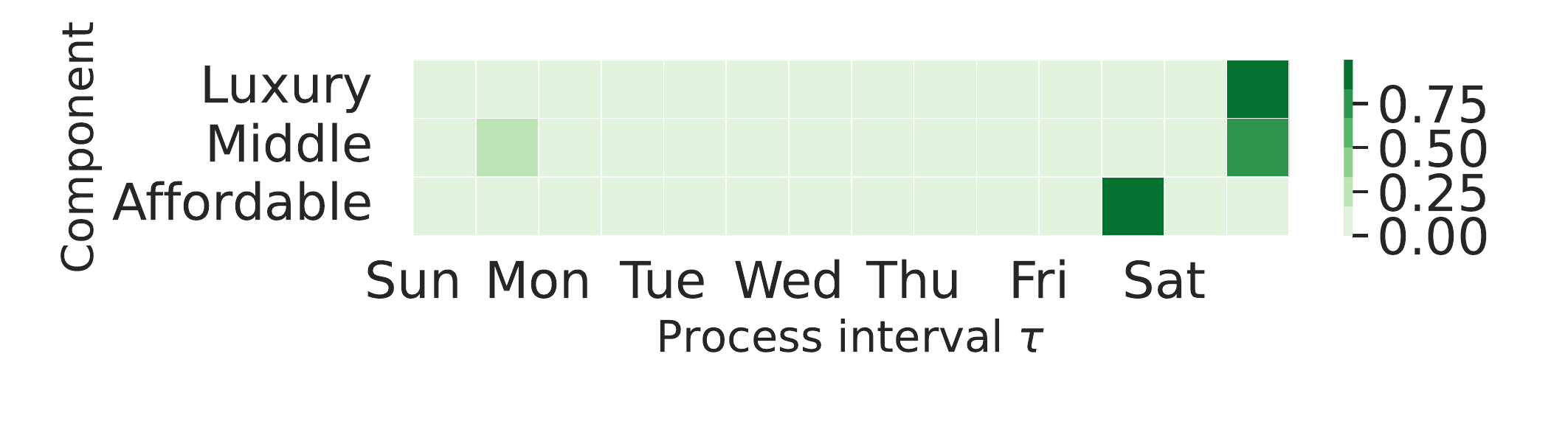}\\
            \vspace{-1em}
            (b-iii) Regime\#3 (Memorial Day sale) \\
        \end{minipage}
        &
        \hspace{1.2em}
        \vspace{-0.5em}
    \centering
    \renewcommand{\arraystretch}{-0.5}
    \hspace{-3em}
    \begin{tabular}{cccc}
    \includegraphics[height=60pt]{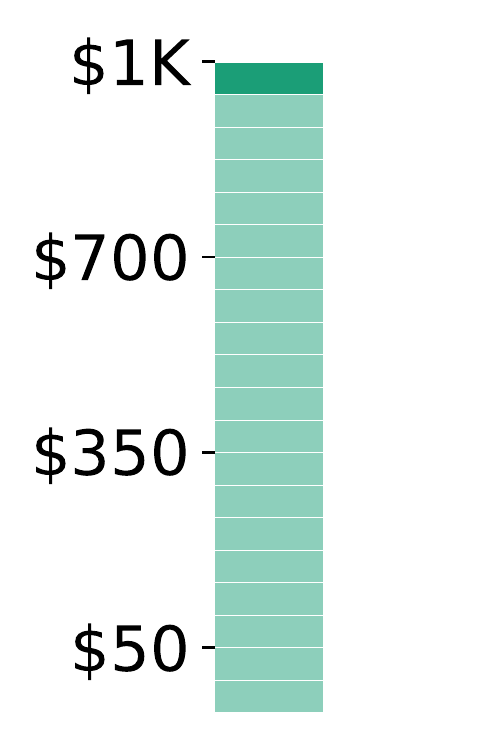} &
    \hspace{-30pt}    
    \includegraphics[width=0.37\columnwidth]{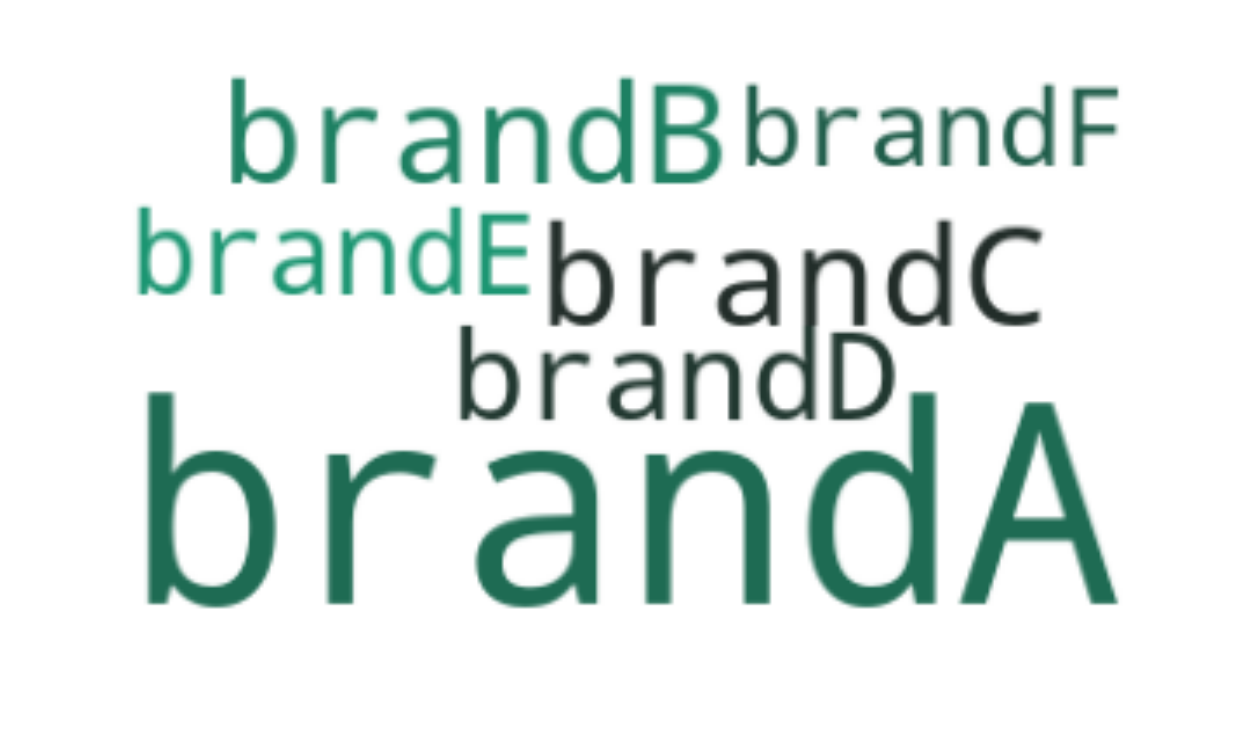} 
    &
    \hspace{-20pt}
    \includegraphics[width=0.37\columnwidth]{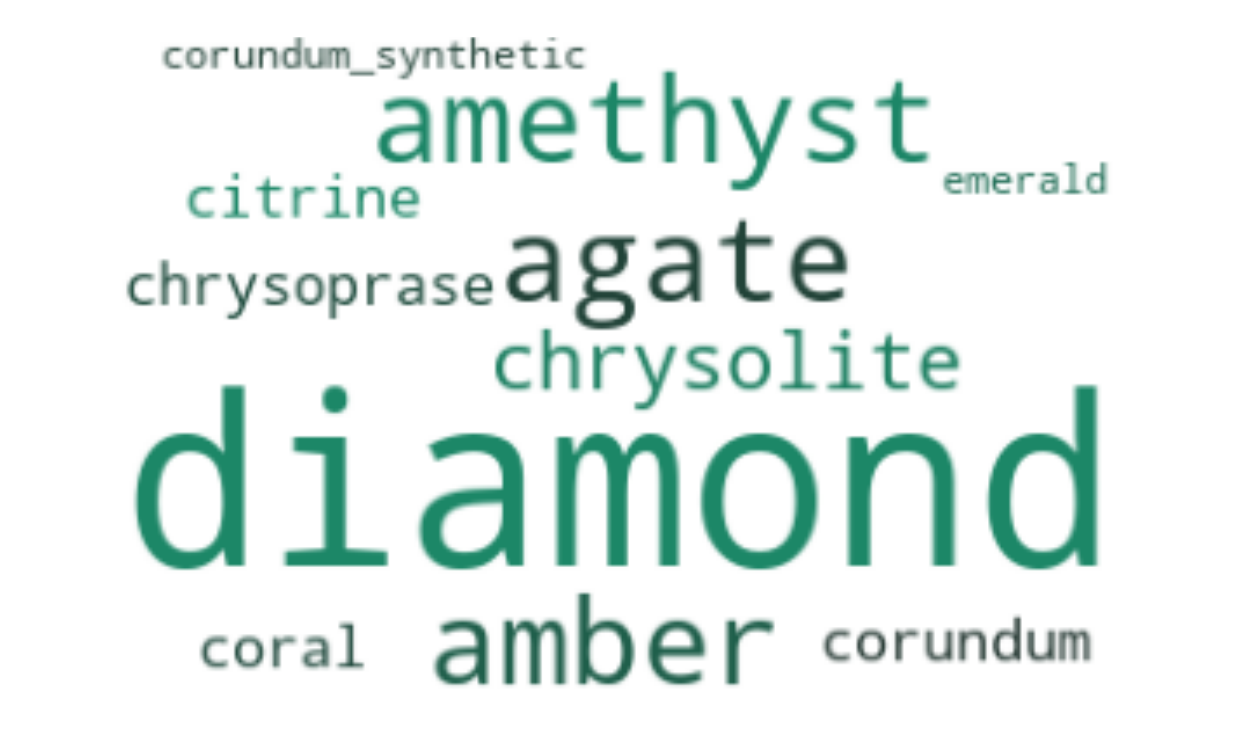}
    &  
    \hspace{-20pt}
    \includegraphics[width=0.37\columnwidth]{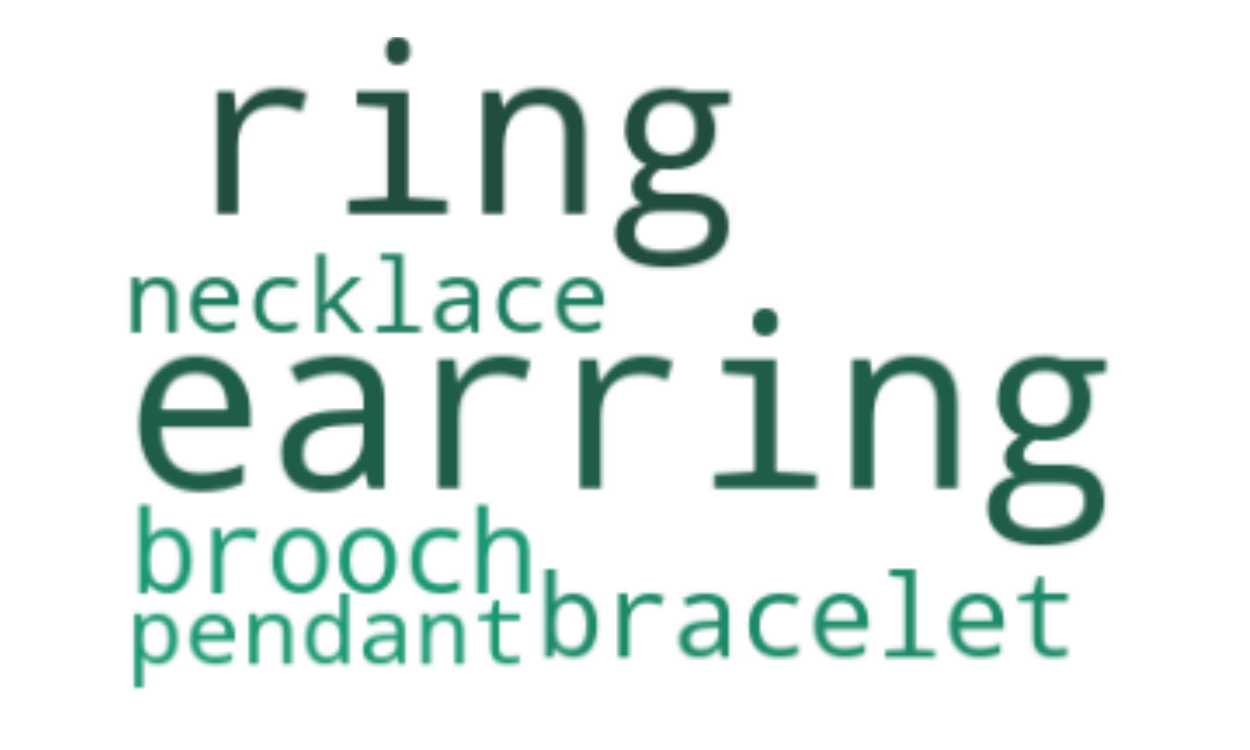}
    \vspace{-0.3em}
    \\
    & & 
    \hspace{-6em} 
    \vspace{-0.5em}
    (c-i) Luxury component & \\
    \includegraphics[height=60pt]{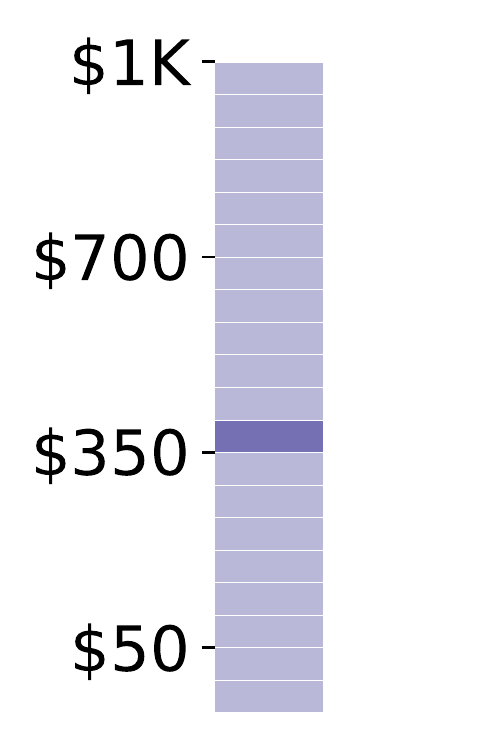} &
    \hspace{-30pt}    
    \includegraphics[width=0.37\columnwidth]{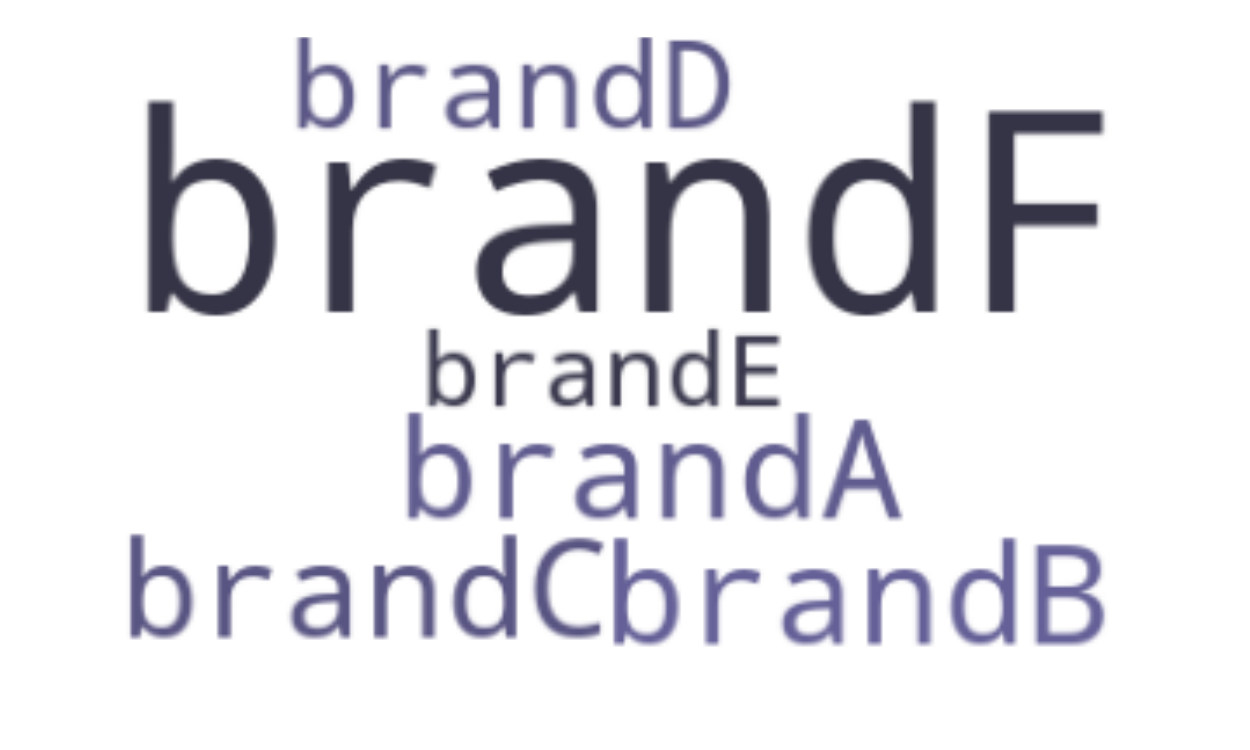} 
    &
    \hspace{-20pt}
    \includegraphics[width=0.37\columnwidth]{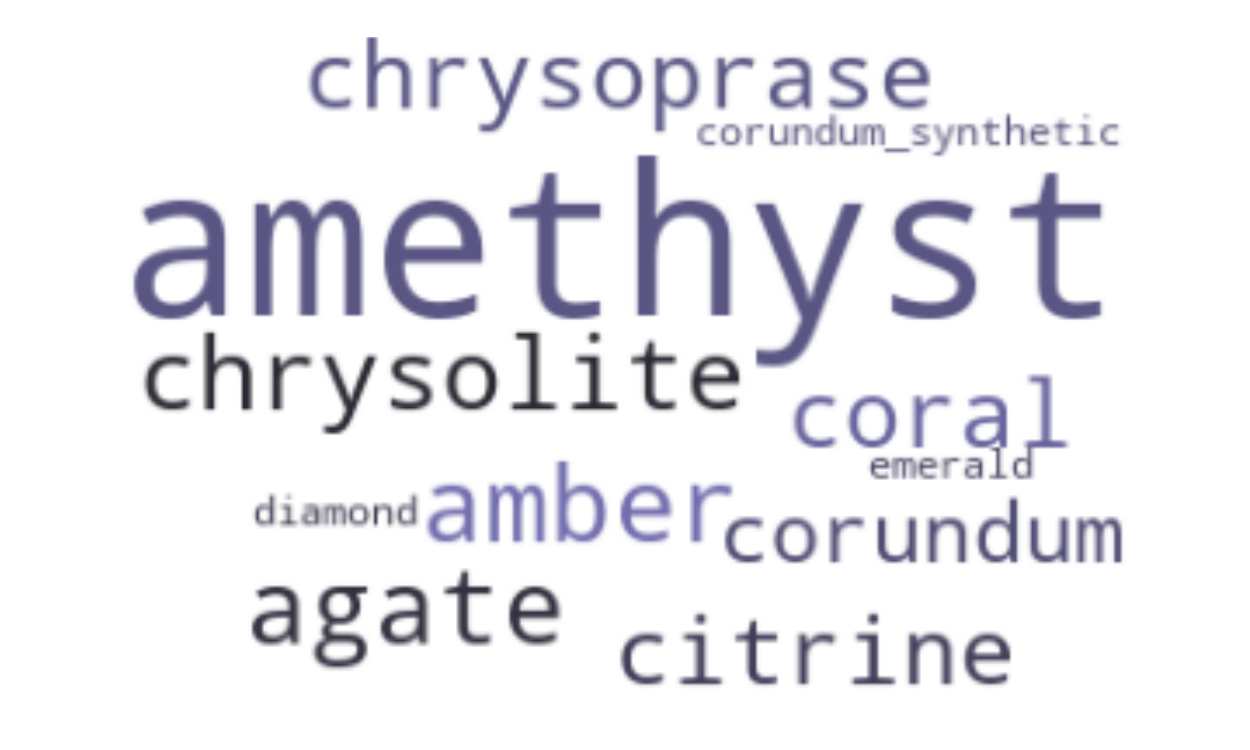}&  
    \hspace{-20pt}
    \includegraphics[width=0.37\columnwidth]{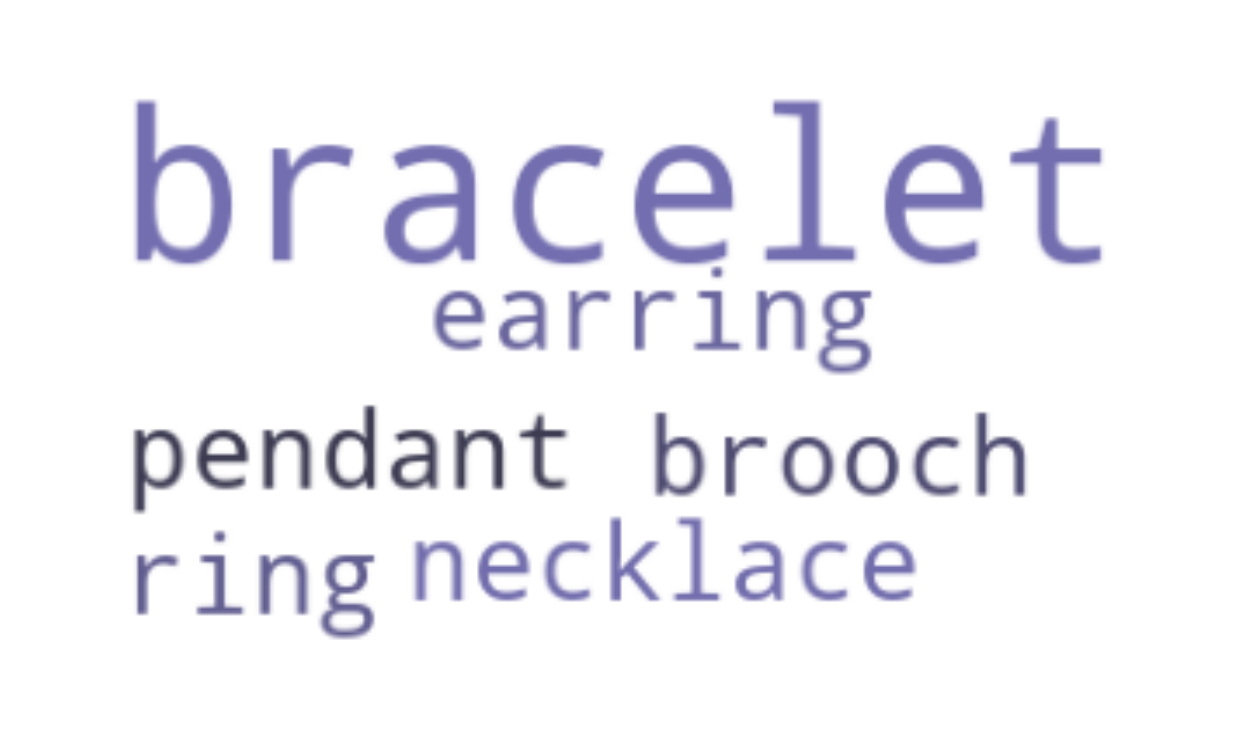}
    \vspace{-0.3em}
    \\
    & & \hspace{-6em}
    \vspace{-0.5em}
    (c-ii) Middle component & \\
    \includegraphics[height=60pt]{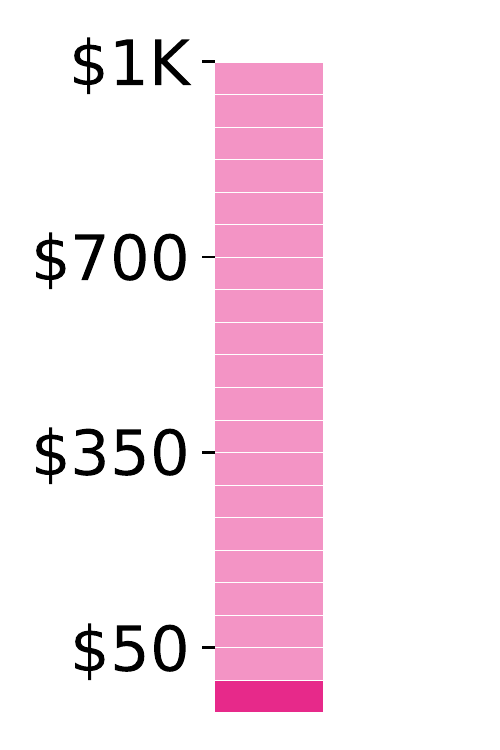} &
    \hspace{-30pt}
    \includegraphics[width=0.37\columnwidth]{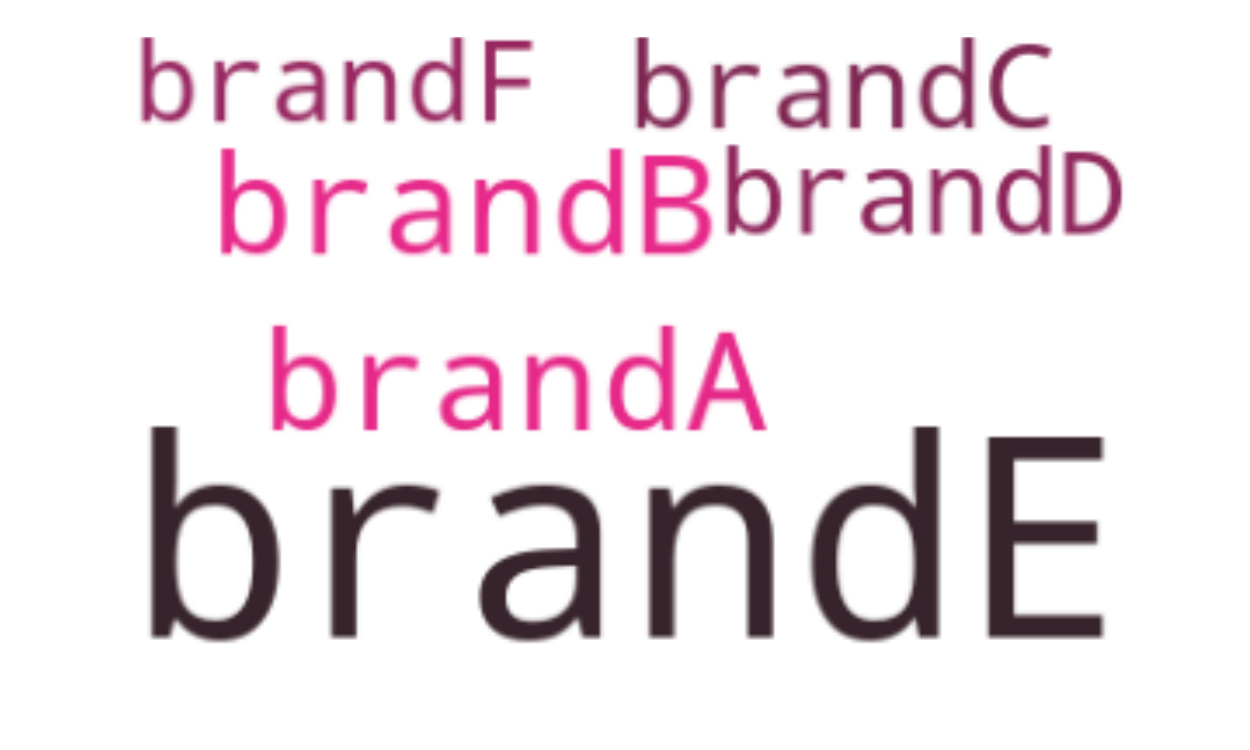}
    &
    \hspace{-20pt}
    \includegraphics[width=0.37\columnwidth]{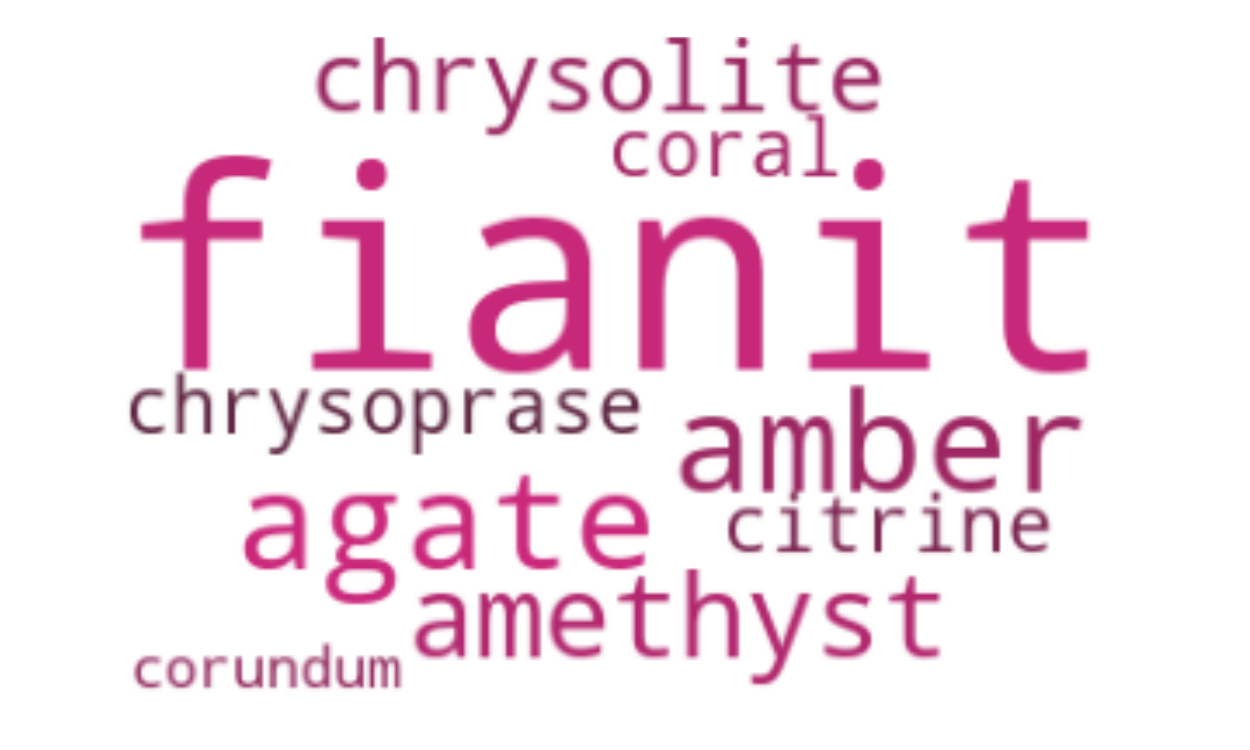}&  
    \hspace{-20pt}
    \includegraphics[width=0.37\columnwidth]{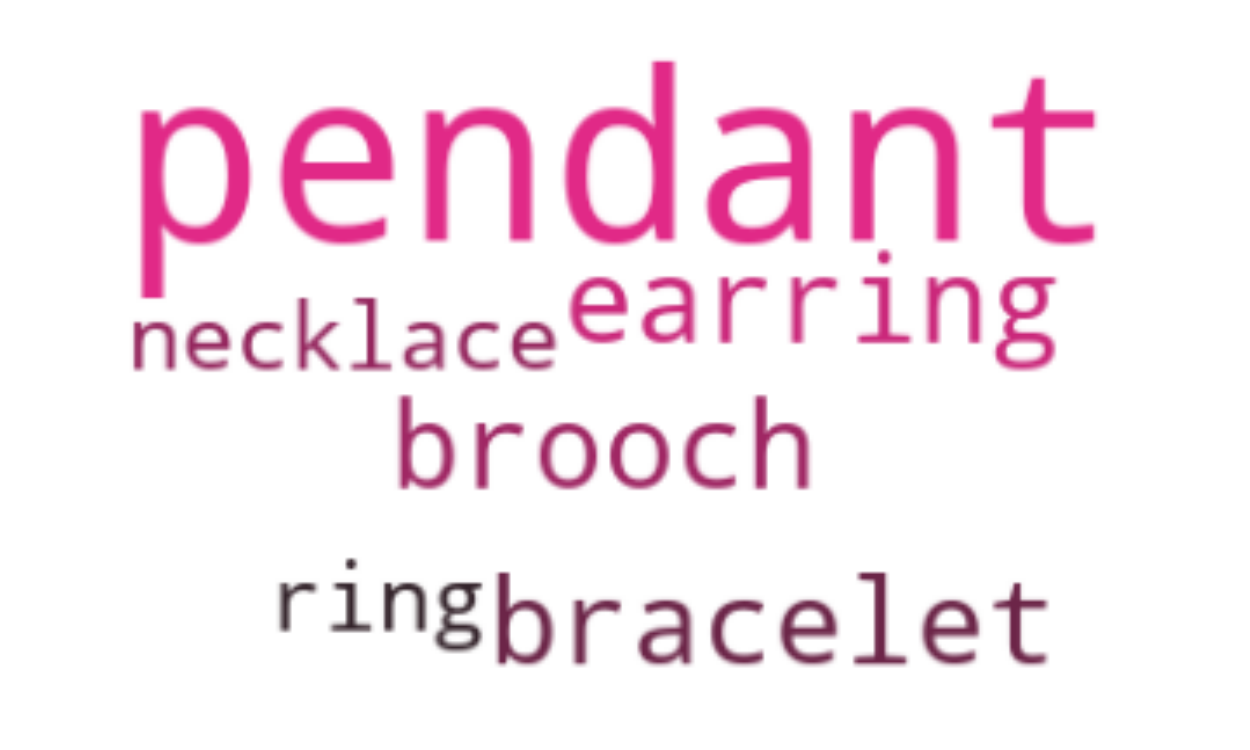}
    \vspace{-0.3em}
    \\
    & & \hspace{-4em} 
    \vspace{-1.5em}
    (c-iii) Affordable component &
    \end{tabular}
    \renewcommand{\arraystretch}{1}
    \end{tabular}
    \vspace{-0.7em}
    \caption{
    Market analysis of \method on the \jewelry dataset:
    (a) \method adaptively captured the changes in purchase behaviors 
    caused by
    the sale.
    (b) The three \topics (Luxury, Middle, Affordable) for time.
    A darker color denotes a stronger relationship between each \topic and the time.
    It shows when the \topics attract consumer interest.
    (c) The three \topics for each \attribute (price/brand/gem/accessory type) in Regime $\#1$.
    Each of the columns shows four \attributes .
    A darker color in the price rank and a larger size in the word cloud 
    denote a stronger relationship with the \topic .
    }
    \label{fig:jewelry_topics}
    \hspace{-1em}
    \vspace{-1em}
    \centering
    \includegraphics[width=\linewidth]{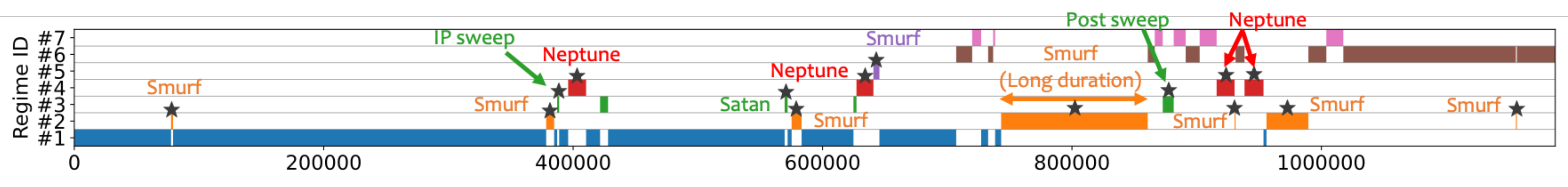}
    \vspace{-0.5em}
    \caption{Real-time intrusion detection of \method 
    on \kddcup dataset:
    the stars indicate intrusions.
    It successfully 
    identified the multiple types of intrusions of different durations
    (i.e., Regime~\#2: Smurf, 
    Regime~\#3: Probe attacks,
    Regime~\#4: Neptune).
    }
\vspace{-1.5em}
\label{fig:anom_effectiveness}
\end{figure*}
 }

In this section, 
we evaluate the performance of \method.
We answer the following questions through the experiments.
\begin{itemize}
    \item[(Q1)]
    \textit{Effectiveness}:
        How successfully does \method discover 
        compact description $\cand$ in given tensor streams?
    \item[(Q2)]
        \textit{Accuracy}:
        How accurately does 
        it achieve modeling, clustering, and streaming anomaly detection?
    \item[(Q3)]
    \textit{Scalability}:
        How does it scale in terms of computational time?
\end{itemize}

\vspace{-0.2em}
\myparaitemize{Datasets \& Experimental Setup}
We use eight 
real datasets and four synthetics.
The real datasets contain the event tensor streams of local mobility, e-commerce,
and network traffic/intrusion
and are summarized in Table~\ref{table:datasets}.
The synthetics and experimental settings are described in
Appendix~\ref{sec:app:expsetup}.

\myparaitemize{Baselines}
Our experiments are evaluated with twelve baselines.

Probabilistic generative models:
\bit
\item 
Latent Dirichlet Allocation (LDA) \cite{blei2003latent}
- A classical topic model, 
where the topic distribution is a multinomial.

\item
Neural Topic Model (NTM) 
\cite{WangLCKLS19}
- A topic model based on 
neural variational inference.
\item
TriMine \cite{MatsubaraSFIY12}
- A factorization method for a high-order tensor,
whose entries consists of multiple attributes and 
a timestamp.
\eit

Clustering approaches for time series, tensor, and data streams:
\bit
\item
K-means 
- The standard K-means clustering algorithm using Euclidean distance.
\item
TICC \cite{hallac2017toeplitz} 
- A clustering method for multivariate time series, 
where each cluster is characterized by a correlation network.
\item 
CubeMarker \cite{honda2019multi}
- An offline approach for discovering distinct patterns in tensor time series.
\item 
Time-Aware LSTM (T-LSTM) \cite{DBLP:conf/kdd/BaytasXZWJZ17}
- A time series clustering method for sequences with irregular time intervals.
\item 
DBSTREAM \cite{hahsler2016clustering}
-  A clustering algorithm for evolving data streams, 
which incrementally updates the density of clusters.
\eit

Unsupervised anomaly detection methods:
\bit
\item
Local Outlier Factor (LOF)
\cite{10.1145/342009.335388}
- A density-based method for a collection of data points.
\item
Isolation forest (iForest)
\cite{liu2012isolation} 
- An offline method, 
where 
a forest of random cuts of data points 
isolates outliers.
\item
Robust Random Cut Forest (RRCF)
\cite{10.5555/3045390.3045676}
- A tree-based approach, which is designed for use with streaming data.
\item
MemStream \cite{DBLP:conf/www/0001JSKH22} 
- A streaming approach using a denoising autoencoder and a memory module.
\eit
\vspace{-1em}
\subsection{Q1.Effectiveness}
\label{sec:effectiveness}
We first demonstrate how effectively \method works on real datasets.
Please also see the results in \electronics
in Appendix~\ref{sec:app:expeffect}.

\myparaitemize{Local Mobility}
The results for \ytaxi have already been presented in \autoref{fig:preview}.
As already seen, 
\method identifies multiple regimes and their shifting points (\autoref{fig:preview}~(a)),
and captures latent \topics (\autoref{fig:preview}~(b)(c)).
These patterns reflect complicated social conditions 
and help us to understand human activities.
 
\myparaitemize{Online Marketing Analytics}
\autoref{fig:jewelry_topics} shows our mining result 
for the \jewelry dataset.
This dataset is an e-commerce-log collected 
from an anonymous jewelry store.
Each of the logs consists of four attributes,
namely 20 prices, 6 anonymous brands, 32 gems, and 8 accessory types,
with 12-hour timestamps.
The price attribute is defined every fifty dollars 
up to 1K dollars i.e., 20 stages.
\bit
\item
\myiparapara{Regime identification}
As shown in \autoref{fig:jewelry_topics}~(a), 
\method 
generates Regime~$\#1$ 
and starts monitoring the tensor stream.
In late November,
it detects a regime transition and generates Regime~$\#2$.
Similarly,
in late May,
it generates a new Regime~$\#3$.
These periods coincide with 
Black Friday
~\footnote{Black Friday is a big sale event on the 4th Friday of November.}
and Memorial Day
~\footnote{On Memorial Day, most jewelry shops hold special sales.}.
This suggests 
our method captures the change in purchase behaviors 
caused by
the sale.
\item
\myiparapara{Multi-aspect \topic analysis}
\autoref{fig:jewelry_topics} (b)(c) shows three 
\topics,
which we manually named ``Affordable'', ``Middle'', and ``Luxury''.
First, 
\autoref{fig:jewelry_topics} (b) shows the three \topics 
for the time attribute (i.e., $\Mtime$) 
in each regime.
It demonstrates 
when each \topic attracts consumer interest.
\autoref{fig:jewelry_topics} (c) shows
the three \topics for each \attribute
(i.e., $\{\Matt^{(\lmode)}\}_{\lmode=1}^4$) in Regime $\#1$.
The \topics reveal 
the latent groups in four \attributes (i.e., price, brand, gem, accessory type).
It shows that
each \topic (row) is strongly 
related to different brands or accessories.
\eit
Here, we provide the reader with some application scenarios.
For targeted advertising and promotion strategies,
analysts investigate purchase logs 
with millions, billions 
or even trillions
\cite{rakthanmanon2012searching} of events.
However, this approach 
requires expert knowledge and time resources.
Since \method can automatically and efficiently summarize 
a massive amount of data into just a handful of \topics,
it could provide analysts with a summary of the market or user preferences.
Also, 
it is a more critical issue to analyze
how the purchase behaviors change due to fads and sales.
Our method recognizes the changes in dynamics as regime transitions
and adaptively generates the summary for each regime.

\TSK{
\begin{figure*}
\vspace*{-1.5em}
    \centering
    \hspace*{-0.08cm}
    \includegraphics[width=2.25\columnwidth]{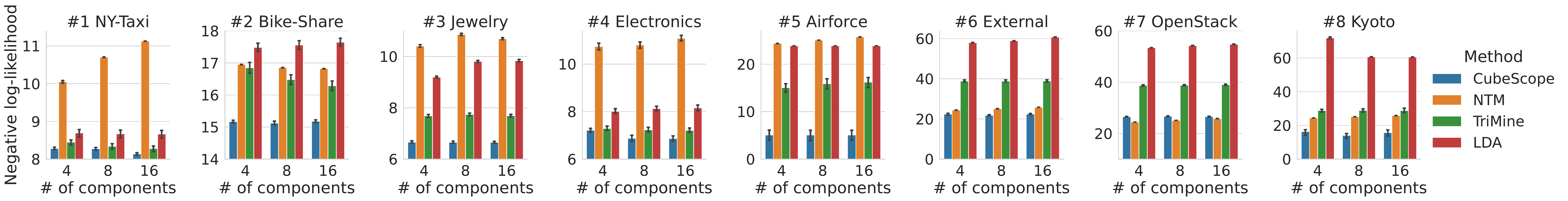}
    \centering
    \vspace{-2.5em}
    \caption{
        Modeling accuracy of \method:
        the method consistently 
        outperforms 
        its baselines (lower is better).
        }
    \label{fig:perplexity}
\vspace{-1.5em}
\end{figure*}
}
\TSK{
\begin{figure} 
    \includegraphics[width=1\columnwidth]{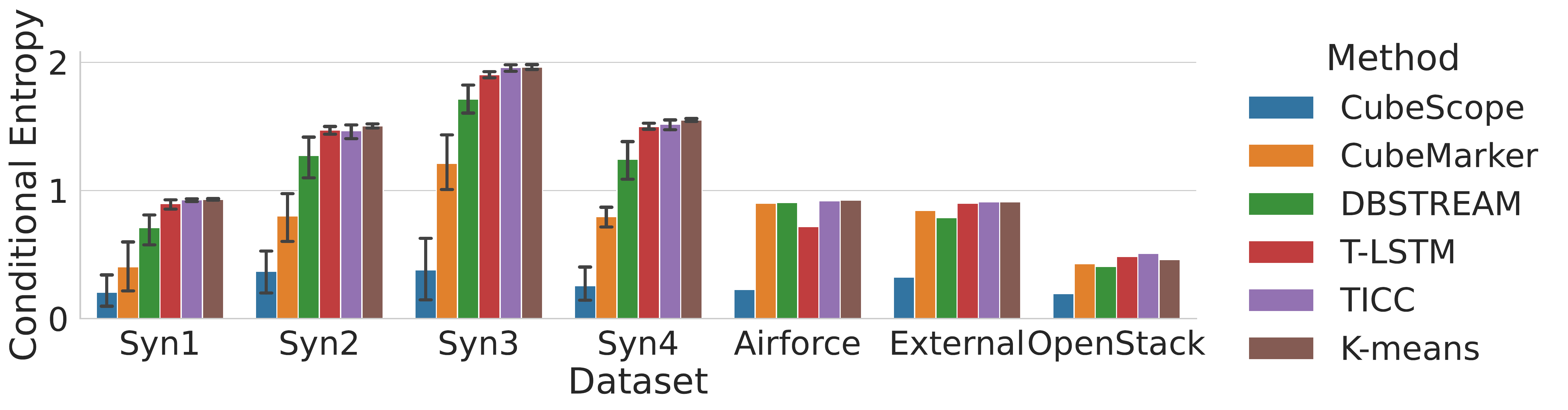}
    \vspace{-2.7em}
    \caption{ 
        Clustering accuracy
        with respect to conditional entropy (lower is better). 
    }
    \label{fig:seg_f1}
    \begin{tabular}{cc}
    \begin{minipage}{0.38\columnwidth}
    \hspace{-0.9em}
    \includegraphics[width=\linewidth]{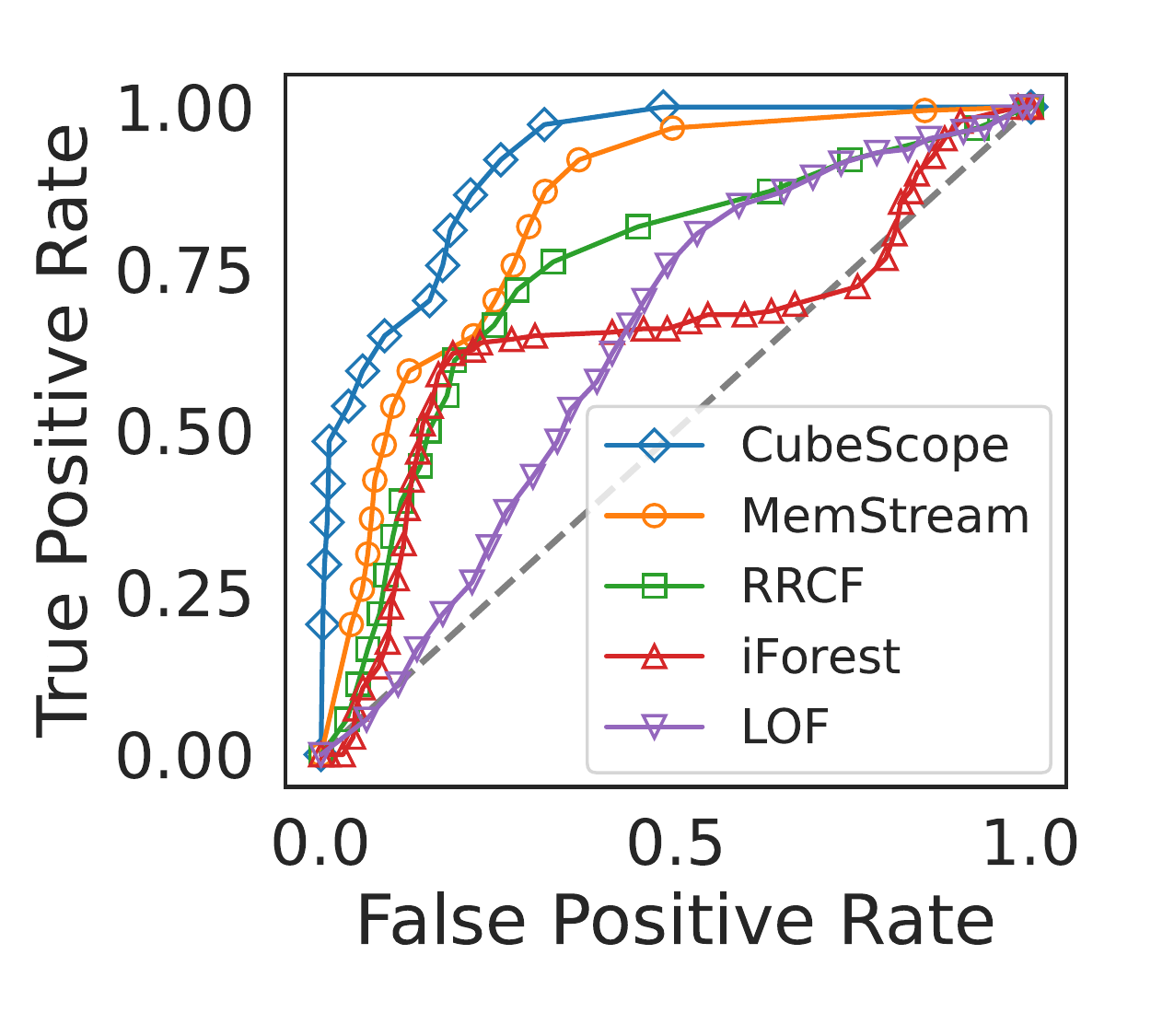}
    \end{minipage}
    &
    \begin{minipage}{0.68\columnwidth}
    \hspace{-2.3em}
    \vspace{-1.6em}
    \includegraphics[width=\linewidth]{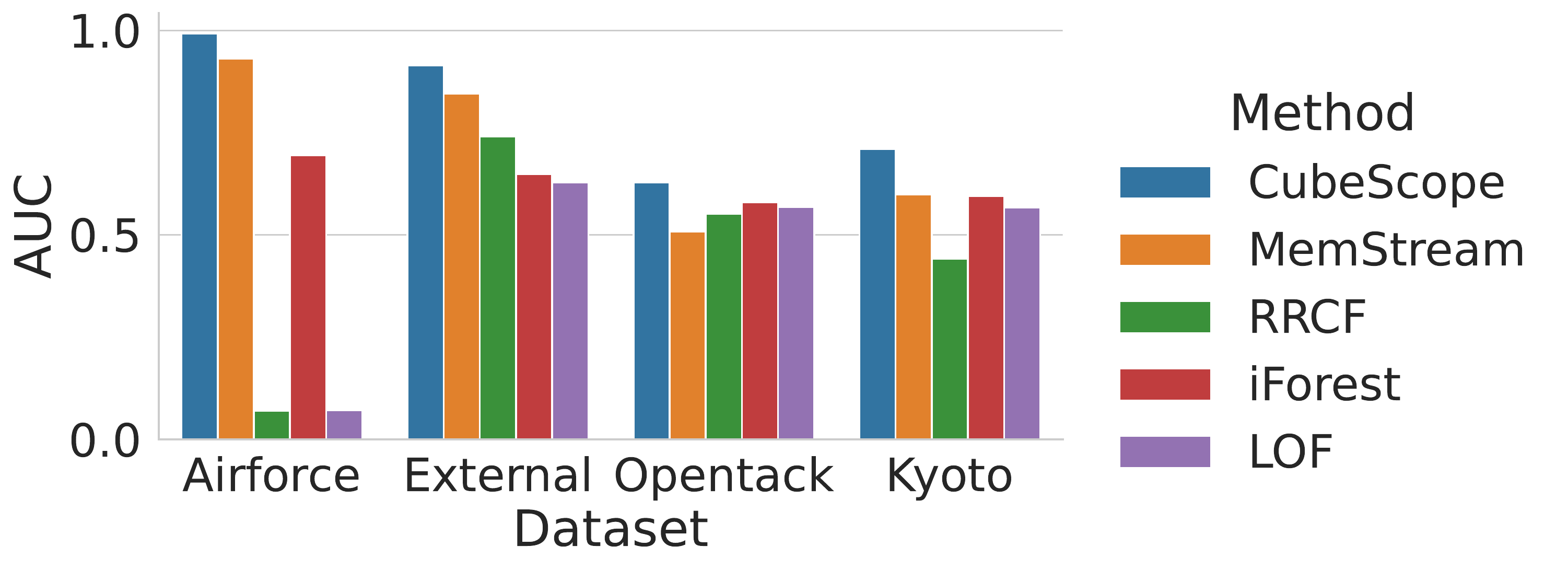}
    \end{minipage}
    \end{tabular}
    \vspace{-2em}
    \caption{
    Detection accuracy (higher is better):
    \method consistently wins.
    (left) The ROC curve on \ciexternal dataset.
    (right) The ROC-AUC on all datasets.
    }
    \label{fig:roc_auc}
    \begin{tabular}{cc}
    \hspace{-1.1em}
        \includegraphics[height=0.33\columnwidth]{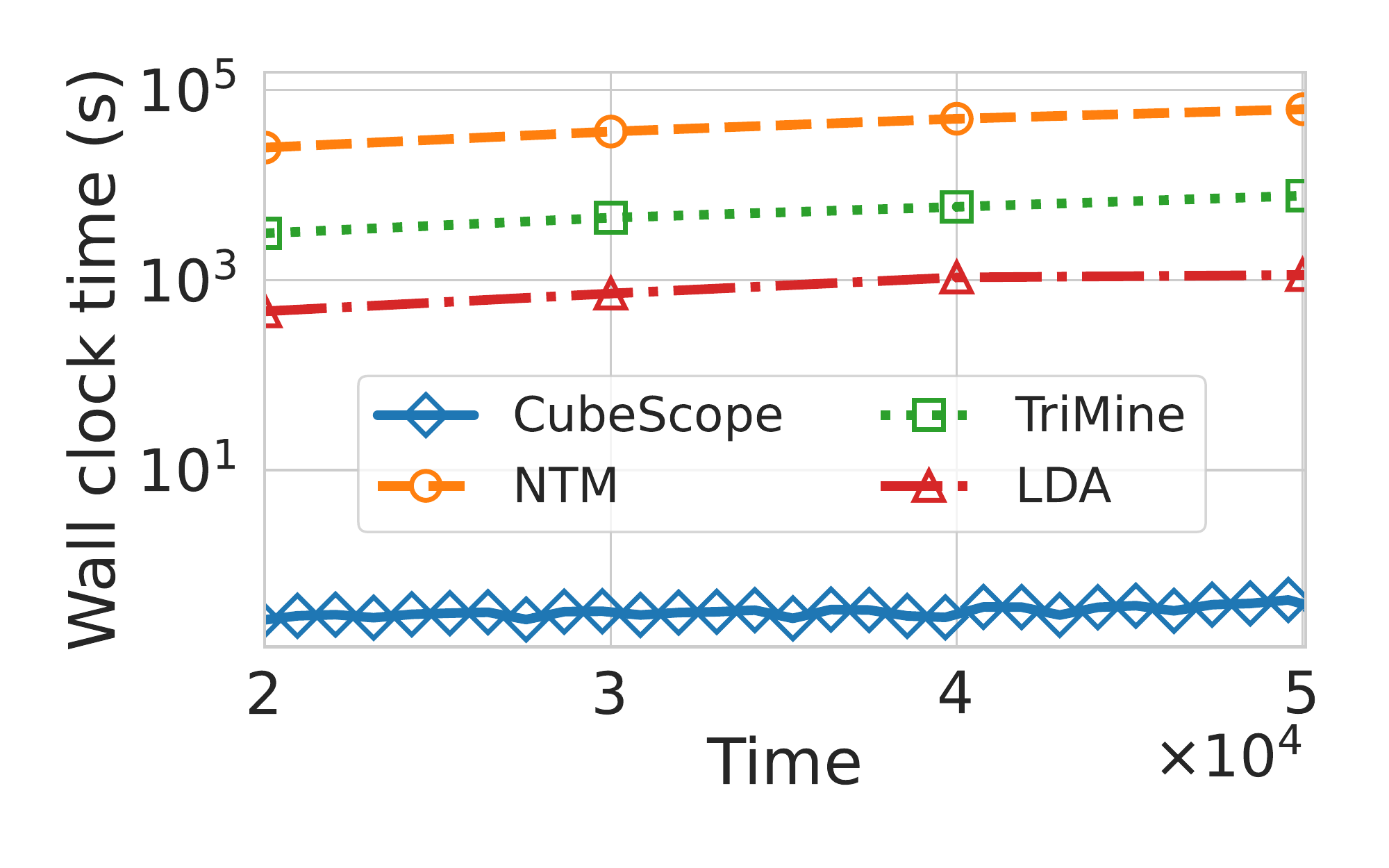} &
    \hspace{-2em}
        \includegraphics[height=0.33\columnwidth]{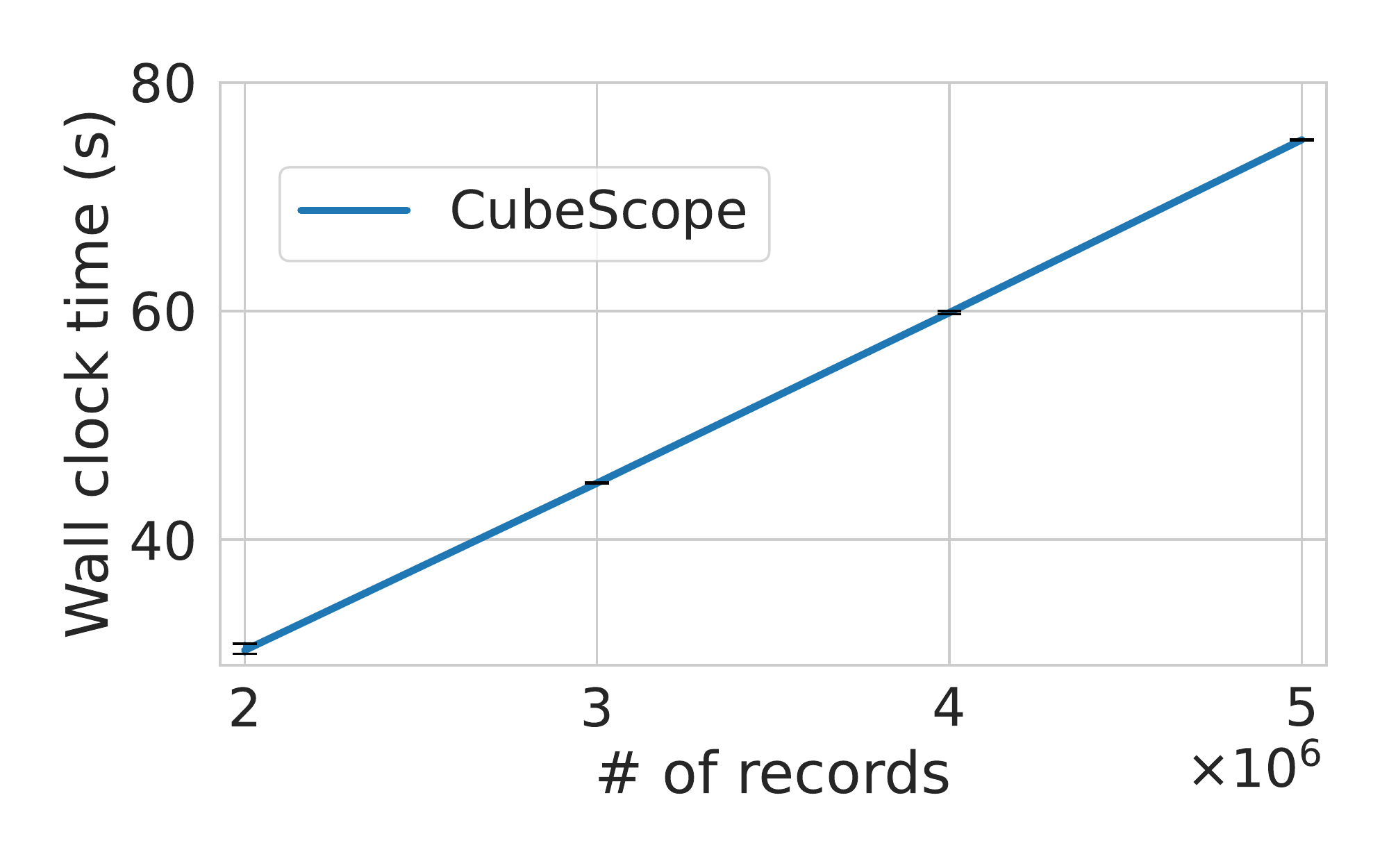}
    \end{tabular}
    \vspace{-2.3em}
    \caption{
        Scalability of \method : 
        (left) Wall clock time vs. data stream length.
        \method surpasses its competitors at any time.
        It is up to 312,000x faster than the baselines.
        (right) Average wall clock time vs. \# of records in a process.
        The algorithm scales linearly.
        }
    \label{fig:time}
\end{figure}
 }

\myparaitemize{Cybersecurity}
We demonstrate 
the real-time intrusion detection of \method .
\autoref{fig:anom_effectiveness} shows the result 
for the \kddcup dataset,
which contains multiple intrusions simulated in a military 
network environment within $1.2$ million records.
We investigated the detected regimes and 
found
that most corresponded to actual intrusions.
For example,
Regime~\#2~(orange) and Regime~\#4~(red) 
correspond to Smurf and Neptune attacks, respectively.
Regime~\#3~(green) captures IP sweep/Stan/Port sweep.
These intrusions are categorized 
as probe attacks \cite{serinelli2020training}.
Most importantly, 
these anomalies arise over time and thus their numbers, durations, and features are unknown in advance,
whereas
\method is fully automatic.
It
automatically recognizes anomalies and their types
while updating the information for each type of anomaly in a streaming setting.
We also conducted a quantitative analysis of this result
in terms of the clustering and anomaly detection in Section~\ref{subsec:accuracy}.

\subsection{Q2. Accuracy}
\label{subsec:accuracy}
We next evaluate the accuracy of \method
in terms of modeling, clustering, 
and anomaly detection.

\myparaitemize{Modeling}
We compared the modeling accuracy of \method and 
the probabilistic generative models.
\autoref{fig:perplexity} shows the average negative log-likelihood 
of every current tensor $\tensorC$ of length $168$ 
for each model.
For a fair comparison, we use 4, 8, and 16 \topics/topics for all models.
A lower value indicates better model construction.
Unsurprisingly,
\method achieves high modeling accuracy on all datasets
because it can capture high-order tensor streams.
Since LDA and NTM handle a tensor as a large matrix,
they cannot capture multi-aspect features.
TriMine is designed for a high-dimensional and sparse tensor, 
but it cannot capture tensor streams containing 
various 
time-evolving patterns (i.e., regimes).

\myparaitemize{Clustering}
Next, we show how accurately \method can find regimes.
We use both labeled datasets and synthetics 
because it is insufficient to evaluate only real datasets 
containing some time-evolving patterns 
that are not repeated (i.e., a few clusters appear once).
We generated four types of synthetics as follows \cite{hallac2017toeplitz} 
(see Appendix~\ref{sec:app:expsetup} for details),
and evaluated them ten times and reported the mean 
and standard deviation values.
Finally, we compare a standard measure of conditional entropy (CE)
from the confusion matrix (CM) of the prediction regime labels 
against true cluster labels.
The CE score shows the difference between 
two clusters using the following equation:
$CE = - \Sigma_{i,j} 
\frac{CM_{i,j}}{\Sigma_{i,j}CM_{i,j}}
\log\frac{CM_{i,j}}{\Sigma_{j}CM_{i,j}}$.
Note that an ideal confusion matrix must be diagonal, 
in which case $CE=0$.
\autoref{fig:seg_f1} compares \method with clustering methods.
Our method consistently outperforms its competitors
because it can handle 
high-dimensional and sparse tensors. 
TICC failed to capture sparse sequences.
T-LSTM is designed for sequences with sparsity
but cannot handle high-dimensional tensors.
CubeMarker can capture tensor time series
but cannot handle sparse tensors.
DBSTREAM has the ability to 
recognize clusters in data streams
but cannot capture multi-aspect features in tensor streams.

\myparaitemize{Anomaly Detection}
We evaluate anomaly detection performance
for four real datasets containing ground truth anomalies.
We first compute anomaly scores for \method and unsupervised baselines, 
and then select the top-$k$ most anomalous periods
($k=20,40,\ldots$).
Next, we compute true and false positive rates
for each method's output.
\autoref{fig:roc_auc} shows
the ROC curve for \ciexternal dataset 
and ROC-AUC for all datasets.
A higher value indicates better detection accuracy.
\method achieves a high detection accuracy for every dataset,
while other methods cannot detect anomalies very well
because only our approach captures dynamical multi-aspect patterns 
and utilizes them for subsequent anomaly detection.

\subsection{Q3. Scalability}
Finally, we evaluate the computational time needed by \method 
for large tensor streams.
The left part of \autoref{fig:time} shows the wall clock time of 
an experiment performed on a large \ytaxi dataset.
Thanks to the incremental update, our method is independent of data stream length.
In fact, 
our method achieved a constant computation time,
which was up to five orders of magnitude faster than its baselines.
The right part of \autoref{fig:time} shows the computational time 
of \decomposition
when varying the size of an input tensor.
Since \method achieves fast and efficient model estimation 
for $O(\ntotal)$ time (as discussed in Lemma~\ref{lemma:time}),
the complexity scales linearly with respect to the number of events.

%% file: TABLE/table_datasets.tex
\begin{table}[t]
    \footnotesize
    \vspace{-1.2em}
    \caption{Dataset description}
    \label{table:datasets}
    \vspace{-1.5em}
    \begin{tabular}{l|c|c}
        \toprule        
        Dataset  & The form of entry  & Order  \\
        \midrule
        \multicolumn{3}{l}{
        Local Mobility: Ride information \attributes \& timestamp $\rightarrow$ \#rides 
        }\\
        \midrule
        \#1 \ytaxi \cite{nytaxi}
        &  \textit{(Pick-up/Drop-off location ID, Time)} & $3$ \\
        \#2 \nybike \cite{nybike}
        &  \textit{(User's age, Start/End station ID, Time)} & $4$ \\
        \midrule
        \multicolumn{3}{l}{
        E-commerce: Purchase information \attributes \& timestamp $\rightarrow$ \#purchases 
        }\\
        \midrule
        \#3 \jewelry \cite{juwelry}
        &  \textit{(Price, Brand, Gem, Accessory type, Time)} & $4$ \\
        \#4 \electronics \cite{electronics}
        &  \textit{(Brand, Item category, Time)} & $3$  \\
        \midrule
        \multicolumn{3}{l}{
        Network traffic/intrusion: Access detail \attributes \& timestamp $\rightarrow$ \#accesses 
        }\\
        \midrule
        \#5 \kddcup \cite{kddcup}
        &  \textit{(Protocol type, Service, Flag, Land, Duration}
        & $10$  \\
        &\textit{Src/Dst bytes, Wrong fragment, Urgent, Time)} &\\
        \#6 \ciexternal \cite{cidds001}
        &  \textit{(Proto, Src/Dst IP Addr, 
        Src/Dst Pt,}
         & $10$  \\
        &\textit{Flags,Duration,Packets,Bytes, Time)}&\\
        \#7 \ciinternal \cite{cidds001}
        & ''
        & $10$  \\        
        \#8 \kyoto \cite{kyoto}
        &  \textit{(Src/Dst bytes, Count,
        Same srv/Serror/Srv serror rate,}
        & $15$  \\
        &\textit{
        Dst host serror rate/same src port rate/srv serrors rate,} &\\
        &\textit{
        Dst host count/srv count, 
        Duration,Service,Flag,Time)} &\\ 
        \bottomrule
    \end{tabular}
    \normalsize
\end{table}

%% file: 060conclusions.tex
In this paper,
we focused on the dynamic summarization of high-order event tensor streams 
and presented \method, 
which exhibits all the desirable properties
that we listed in the introduction; 
\vspace{-1em}
\bit
\item
\textit{Effective}:
it incrementally captures 
dynamical multi-aspect patterns 
and 
summarizes 
a semi-infinite collection of event tensor streams 
into an interpretable representation.
\item
\textit{General}:
our experiments with various datasets
showed that \method successfully discovers meaningful patterns and anomalies,
and outperforms state-of-the-art modeling, clustering, and anomaly detection
methods.
\item
\textit{Scalable}:
its computational time is constant and independent of the input data length 
and the dimensionality in each \attribute.
\eit
\hide{
\begin{enumerate}
    \item \textit{Effective}:
        our experiments with various types of social activities 
        showed that \method successfully discovers
        time-evolving regimes and latent \topics.
    \item \textit{Interpretable}:
        it is possible to compress complex event streams into 
        a compact summary that describes the key activities. 
    \item \textit{Automatic}:
        it can automatically recognize patterns 
        by exploiting novel coding scheme in a semi-infinite streams.
    \item \textit{Scalable}:
        its computational time is constant and independent of input data length as well as the number of \units in each of the \attributes.
\end{enumerate}
}

%% file: appendix.tex
\section*{Appendix}
\section{Proposed model}
\autoref{table:define} lists the symbols and definitions used in this paper.
All logarithms are to base 2, and by convention we use $0\log0 = 0$. 

\TSK{
\input{TABLE/table_symbol}  
}

\section{Streaming Algorithm}
\label{sec:app:algo}
Algorithm ~\ref{alg:main}\ shows the overall procedure for \methodalgo,
which composed of \decomposition (Algorithm~\ref{alg:decomp}) and \compression (Algorithm.~\ref{alg:compression}). 
\decomposition continuously monitors a current tensor $\tensorC$
and generates a regime $\regime_c$.
Then, \compression updates 
the compact description $\cand$ 
with $\regime_c$ 
and measures the anomalousness of $\tensorC$.

\TSK{
\begin{algorithm}[t]
    \footnotesize
    \caption{\methodalgo $(\tensorC, \cand, Q)$}
    \label{alg:main}
\begin{algorithmic}[1]
    \REQUIRE
        1. Current tensor
        $\tensorC\shapeN^{
            \nunits_1 \times \ldots \times \nunits_{\nmode} \times \tau
        }$\\
        \hspace{1.35em}
        2. Previous candidate solution $\cand = \{\nregime, \regimeset, \nshiftp, \regimeassignment \}$\\
        \hspace{1.35em}
        3. Previous past parameter set $Q$
    \ENSURE
        1. Updated candidate solution  $\cand'$\\
        \hspace{2.05em}
        2. Updated past parameter set $Q'$\\
        \hspace{2.05em}
        3. Anomalousness score $\score{\tensorC}$
    \STATE $\regime_c, Q'$ = \decomposition($\tensorC$, $Q$);
    \STATE $\cand', \score{\tensorC}$
    = \compression($\regime_c, \tensorC, \cand $);
    
    \STATE {\bf return}  $\cand', Q', \score{\tensorC}$;
\end{algorithmic}
    \normalsize
\end{algorithm}

\begin{algorithm}[t]
    \footnotesize
    \caption{\decomposition $(\tensorC, Q)$}
    \label{alg:decomp}
\begin{algorithmic}[1]
    \REQUIRE
        1. Current tensor
        $\tensorC\shapeN^{
            \nunits_1\times\ldots\times\nunits_{\nmode}\times\tau}$\\
        \hspace{1.35em}
        2. Previous past parameter set $Q$  
    \ENSURE
        1. Current model parameter set $\regime_c = \{\Matt^{(1)}, \ldots \Matt^{(\nmode)}, \Mtime\}$
        \\
        \hspace{2.05em}
        2. Updated past parameter set $Q'$

\FOR {{\bf each} iteration}
  \FOR {{\bf each} non-zero element $\Ex$ in $\tensorC$}
    \FOR {{\bf each} entry for $\Ex$}
      \STATE Draw hidden variable $z$; 
      // According to Eq.~(\ref{eqn:gibbs_sampling_time})
    \ENDFOR
  \ENDFOR
\ENDFOR
\STATE Compute $\Matt^{(1)}, \ldots ,\Matt^{(\nmode)}, \Mtime$;
//According to Eq.~(\ref{eqn:normalizedMatrices})
\STATE $\regime_c \leftarrow \Matt^{(1)}, \ldots ,\Matt^{(\nmode)}, \Mtime$; 
\STATE $Q$.deque; \ \ 
// Remove the oldest set of \topic matrices
\STATE $Q' \leftarrow Q$.enque($\regime_c$); \ \ 
// Insert a set of current estimated matrices $\regime_c$
  \STATE {\bf return}  $\regime_c, Q'$ ; 
\end{algorithmic}
    \normalsize
\end{algorithm}

\begin{algorithm}[t]
    \footnotesize
    \caption{\compression $(\regime_c, \tensorC, \cand)$}
    \label{alg:compression}
\begin{algorithmic}[1]
    \REQUIRE
    1. Candidate model parameter set 
    $\regime_c =\{\Matt^{(1)}, \ldots \Matt^{(\nmode)},\Mtime\}$\\
    \hspace{1.35em}
    2. New observation tensor
    $\tensorC\shapeN^{
    \nunits_1 \times \ldots \times \nunits_{\nmode} \times \tau}$\\
    \hspace{1.35em}
    3. Previous compact description $\cand = \{\nregime, \regimeset, \nshiftp, \regimeassignment\}$\\
    \ENSURE
        1. Updated compact description $\cand' = \{\nregime', \regimeset', \nshiftp', \regimeassignment'\}$ \\
        \hspace{2.05em}
        2. Anomalousness score  $\score{\tensorC}$ \\
    \STATE 
    /* Update compact description $\cand$ */
    \IF {$\costT{\tensorC}{\regime_{p}}$ is less than $\costT{\tensorC}{\regime_{c}}$}
    \STATE 
    /* Stay in the previous regime $\regime_p$ */
    \STATE$ \regime_{p} '\leftarrow$ \textsc{Regime update}$(\regime_{p},\regime_{c})$;
    // According to Eq.~(\ref{eqn:model_update})
    \ELSE
        \STATE $\regime_{\alt} = \argmin_{\regime \in \regimeset } \costT{\tensorC}{\regime}$;
        \IF {$\costT{\tensorC}{\regime_{c}}$ is less than $\costT{\tensorC}{\regime_{\alt}}$}
            \STATE 
            /* Shift to the candidate regime $\regime_c$ */
            \STATE $\nregime' \leftarrow \nregime+1$;
            $\regimeset' \leftarrow \regimeset \cup \regime_{c}$;
            \STATE
            $\nshiftp' \leftarrow \nshiftp+1$;
            $\regimeassignment' \leftarrow \regimeassignment \cup (t, \nregime+1)$;
        \ELSE
            \STATE 
            /* Shift to the existing regime $\regime_{\alt}$ */
            \STATE$ \regime_{\alt}' \leftarrow$ \textsc{Regime update}$(\regime_{\alt},\regime_{c})$; 
            // According to Eq.~(\ref{eqn:model_update})
            \STATE
            $\nshiftp' \leftarrow \nshiftp+1$;
            $\regimeassignment' \leftarrow \regimeassignment \cup (t, \alt)$;
        \ENDIF
        
    \ENDIF
    \STATE 
    /* Compute anomalousness score*/
    \STATE $\norm \leftarrow \argmax_{\lregime \in
    \nregime}|\segmentset^{-1}_{\lregime}|$;
    \STATE $\score{\tensorC} \leftarrow \costC{\tensorC}{\regime_{\norm}}$;
    
        \RETURN
        $\cand' = \{\nregime', \regimeset'\, \nshiftp'\, \regimeassignment'\}$,
        $\score{\tensorC}$;\ \ 
    \end{algorithmic}
    \normalsize
\end{algorithm}
}

\myparaitemize{Proof of Lemma \ref{lemma:time}}
\begin{proof}
For each time point,
\methodalgo first runs \decomposition ,
which draws hidden variables $z_{\lunit,\ltime}$ with each entry for non-zero element $x_{\lunit,\ltime}$ in $\tensorC$.
This process requires $O(\#iter\cdot\ntopic\ntotal )$, 
where $\#iter$ is the number of iterations for drawing $z$,
$\ntopic$ is the number of \topics,
and $\ntotal$ is the total number of event entries in $\tensorC$ 
(i.e., $\sum_{\lunit_1} \cdots \sum_{\lunit_{\nmode}} \sum_{\ltime} \Ex_{\lunit_{1},\ldots,\lunit_{\nmode},\ltime}$). 
Since $\#iter$ and $\ntopic$ are small values and constant, they are negligible. 
Thus, the complexity of \decomposition is $O(\ntotal)$.
In \compression, it tracks $\regime_c$ and $\regime_p$.
If it employs the previous regime $\regime_p$ for current tensor $\tensorC$,
it can quickly update the regime,
which requires $O(1)$ time.
Otherwise, it then tries to find the optimal regime in $\regimeset$,
which requires $O(\nregime)$ time.
Overall, \methodalgo needs these two algorithms.
Thus, the complexity is at least $O(\ntotal)$ and 
at most $O(\ntotal + \nregime)$ per process.
\end{proof}

\TSK{
\begin{figure*}[!t]
    \begin{tabular}{ccc|c}
        \renewcommand{\arraystretch}{-0.5}
    \begin{minipage}{0.4\linewidth}
        \centering
        \hspace{-1em}
        \includegraphics[width=0.9\linewidth]{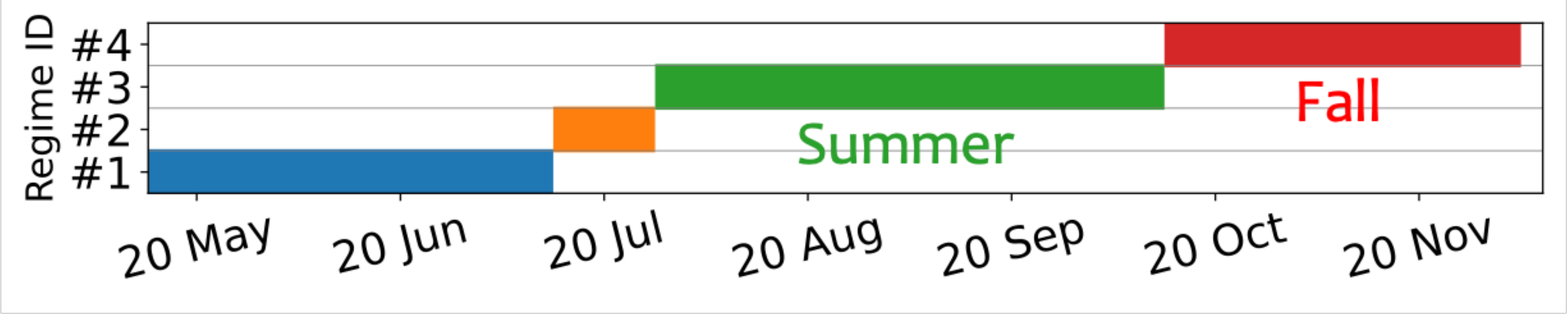}\\
        \vspace{-0.7em}
        (a) Market regimes \\
        \hspace{-1em}
        \includegraphics[width=0.9\linewidth]{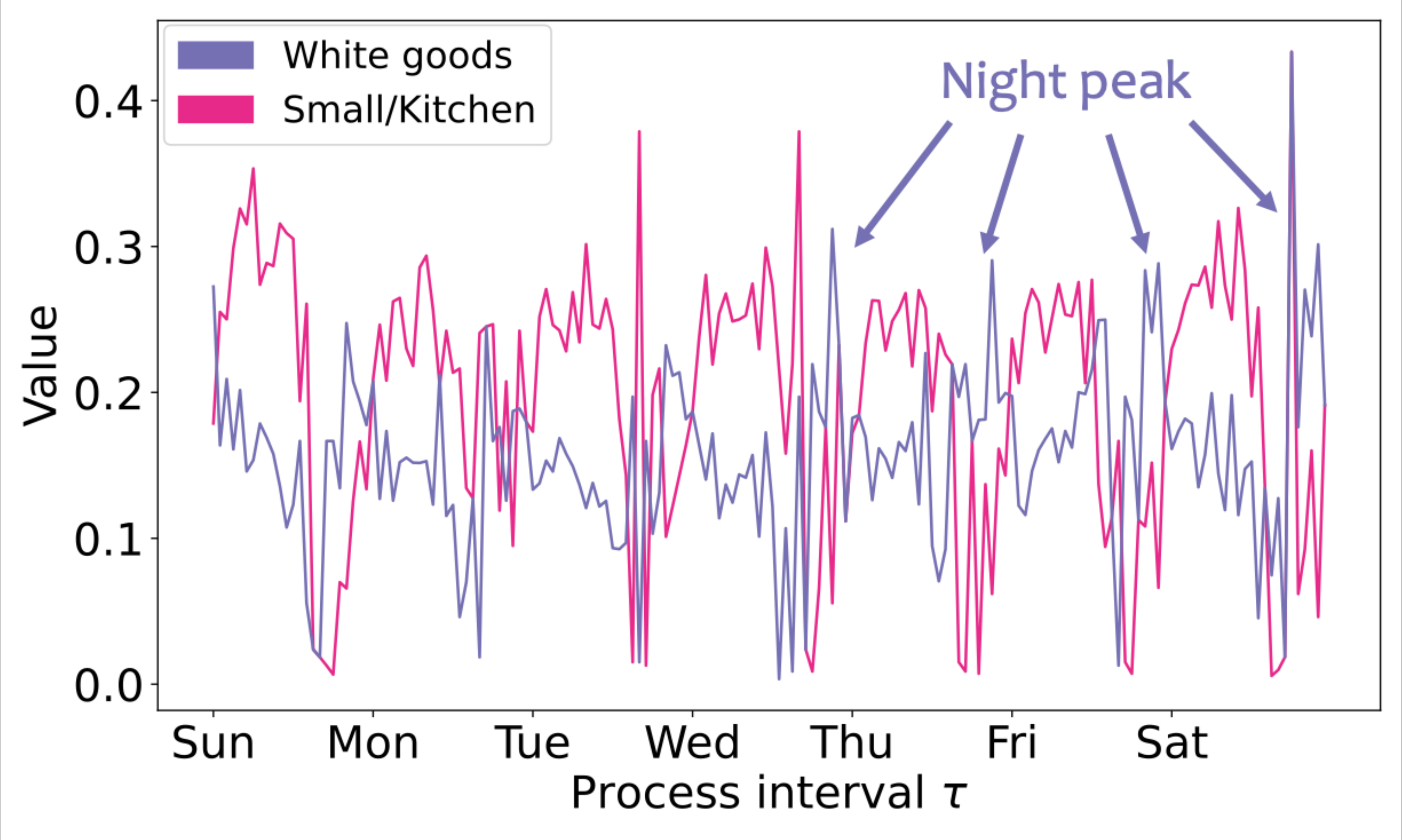}
        \\
        \vspace{-0.7em}
        (b) Time evolution of two major \topics 
    \end{minipage}
    &
    \hspace{-3em}
    \begin{minipage}{0.2\linewidth}
    \centering
    \includegraphics[width=\linewidth]{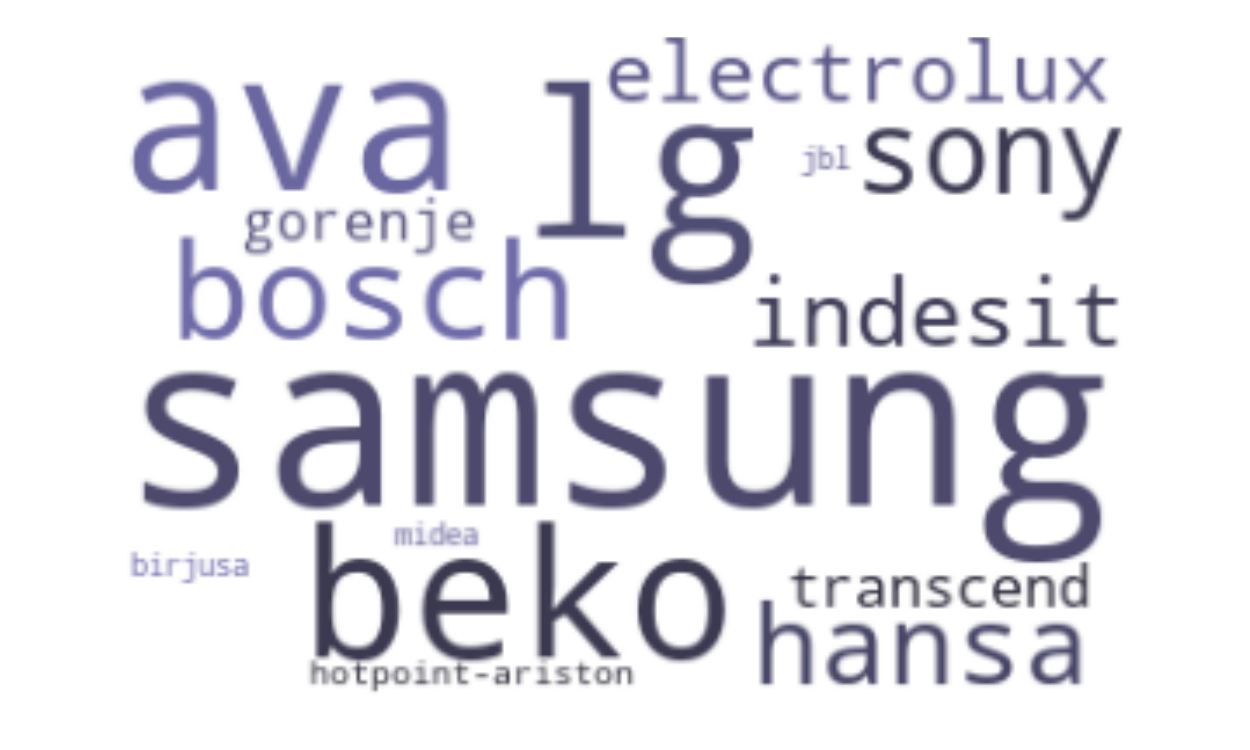}
    \includegraphics[width=\linewidth]{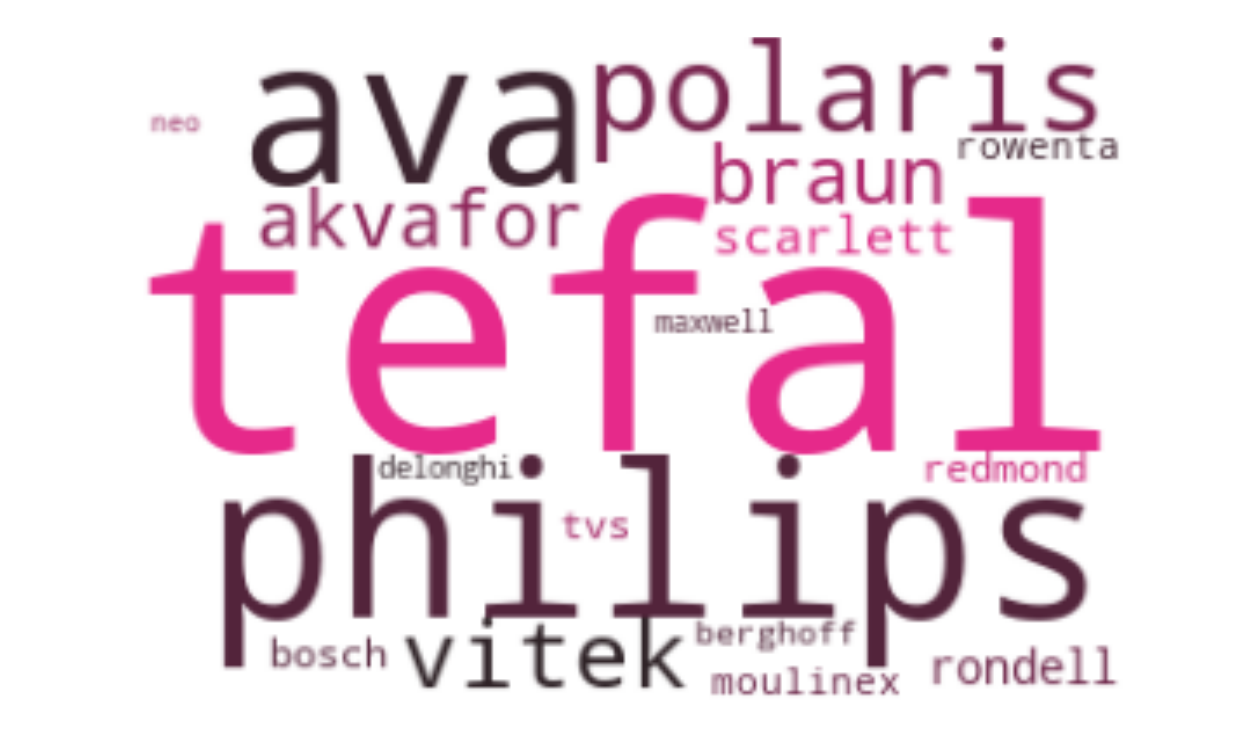}\\
    (c-i) Popular brands in \\White goods (top) and Small/Kitchen (bottom) \topics.
    \end{minipage}
    &
    \hspace{-2em}
    \begin{minipage}{0.2\linewidth}
    \centering
    \includegraphics[width=\linewidth]{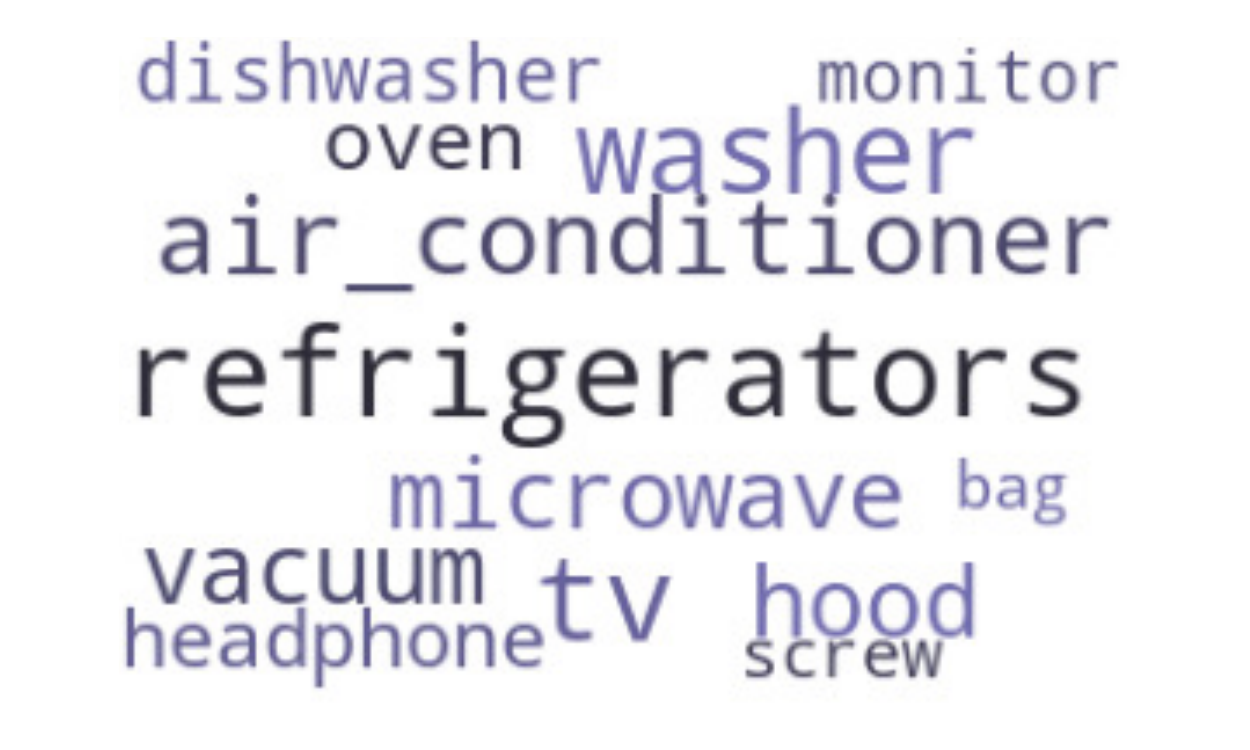}
    \includegraphics[width=\linewidth]{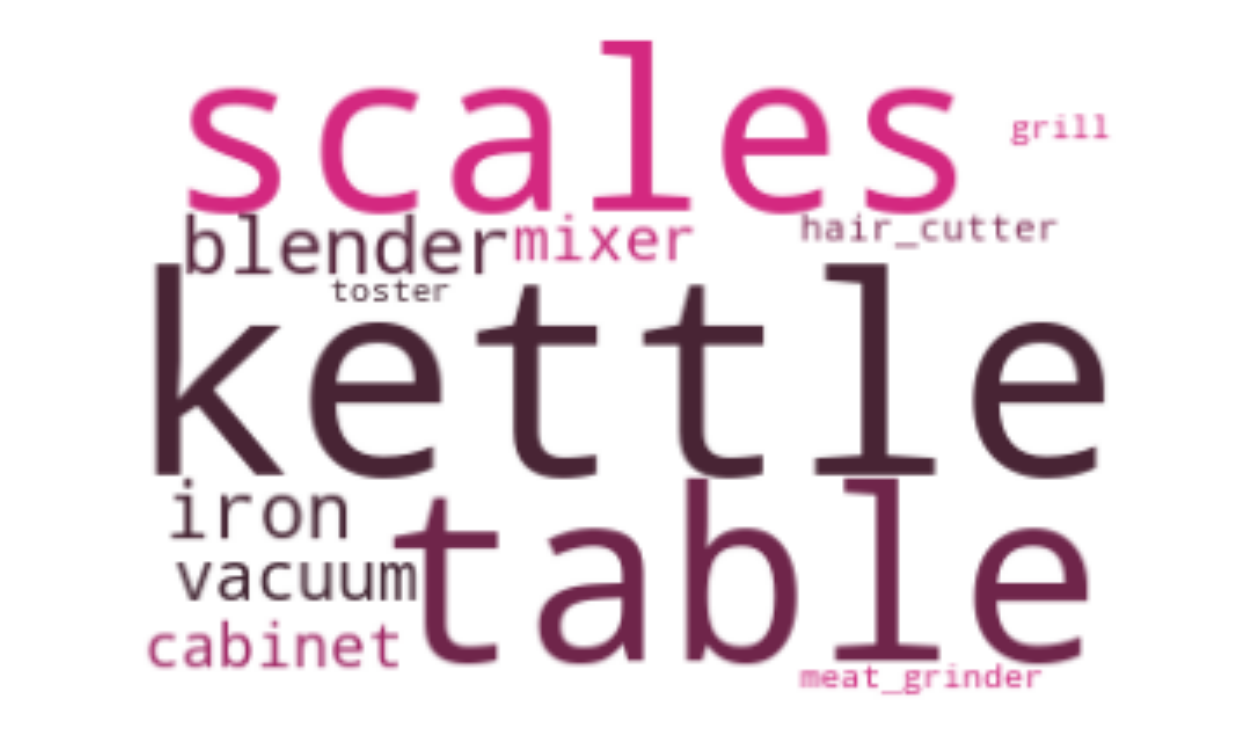}\\
    (c-ii) Popular items in \\White goods (top) and Small/Kitchen (bottom) \topics.
    \end{minipage}
    \hspace{-1.1em}
    &  
    \hspace{-1em}
    \begin{minipage}{0.2\linewidth}
    \vspace{-1.25em}
    \centering
    \includegraphics[width=\linewidth]{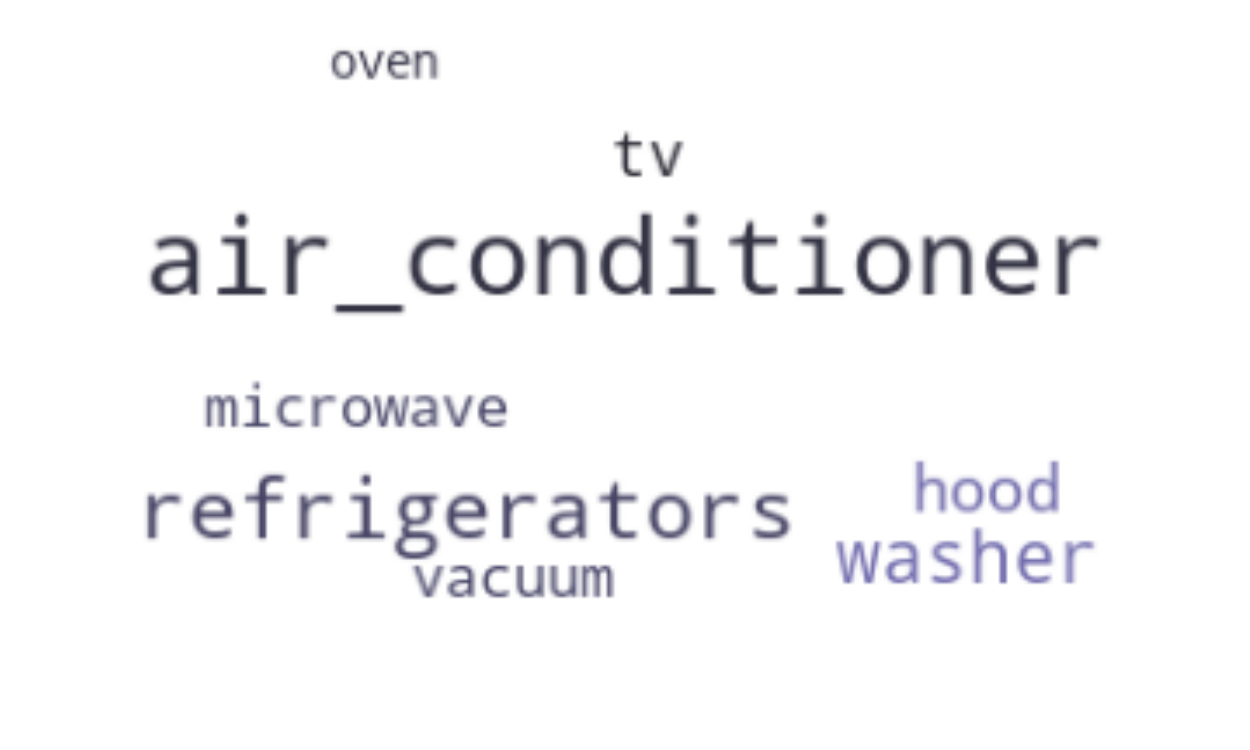} 
    \includegraphics[width=\linewidth]{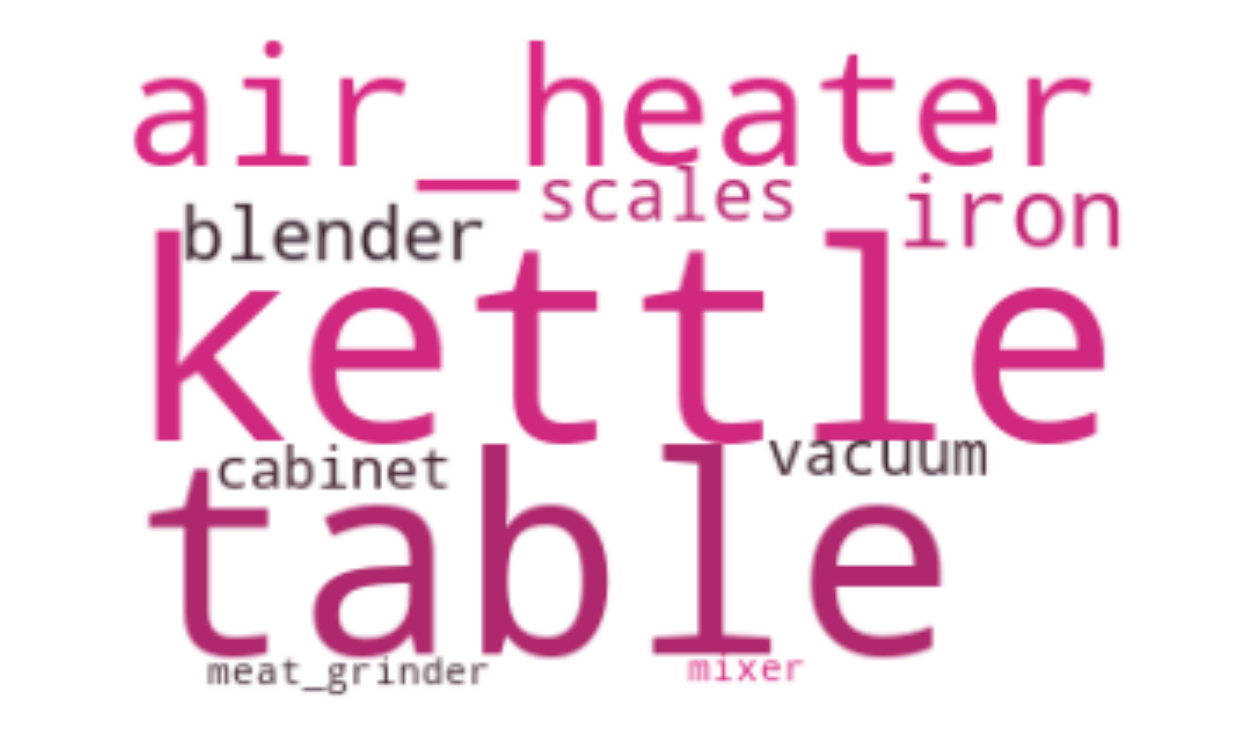} \\
    (d) Changes of popular items 
    with regime transitions
    \end{minipage}
    \renewcommand{\arraystretch}{1}
    \end{tabular}
    
    \vspace{-0.5em}
    \caption{
    Customer behavior modeling on \electronics:
        (a)
        \method discovered a total of four regimes
        that reflect the seasonality of customer behavior.  
        (b)
        Two \topics show clear daily periodicity;
        The White goods \topic spikes at night.
        The Small/Kitchen \topic shows high peaks during the daytime.
        (c)
        Popular brand/item categories for each \topic.
        A larger size denotes a closer relationship.
        (d)
        Popular item categories change with regime transitions.
        Top: air conditioners become the most frequent item in Regime \#2 (early Summer).
        Bottom: air heaters are the popular items in Regime \#4 (Fall).
    }
    \label{fig:electronics}
    \vspace{-1.5em}
\end{figure*}
}

\myparaitemize{Model Initialization}
When we start the iteration 
with regime set $\regimeset$ and $Q$ for past parameter set, 
we uniformly take several sample segments with interval $\tau$
from initial tensor $\tensor_s$ and then estimate the model parameters $\regime_s$ for each.
The most appropriate regime set is determined by monitoring the total encoding cost 
while increasing the number of model parameters $\regime_s$:
\small
\begin{align}
    \label{eqn:costinit}
    \regimeset
        &= \argmin_{\regimeset \in \regimeset_s} \costT{\tensor_s}{\regimeset},
    \end{align}
\normalsize
where $\regimeset_s = \{\regime_{s1}, \regime_{s2}, \dots \}$ is  
a set of regimes estimated from each sample.
We also set 
$L$ to the average interval of the shifting points 
(i.e., the average length of segments).

\myparaitemize{Process Interval $\tau$}
The parameter $\tau$ determines the size of a current tensor, 
as well as the minimum granularity of regimes.
Users need to know regimes under various granularities 
(e.g., daily and weekly patterns),
thus $\tau$ is generally chosen depending on the application.
The runtime of \decomposition scales linearly with the number of records
in a current tensor rather than the size of $\tau$.
Although a larger $\tau$ imposes the algorithm to process a larger current tensor,
there should be a small impact on the runtime 
because we assume sparse tensor streams.

\section{Experiments}
\label{sec:app:exp}

\subsection{Experimental Setup}
\label{sec:app:expsetup}
We conducted our experiments on
an Intel Xeon E$5$-$2637$ $3.5$GHz quad core CPU
with $192$GB of memory and running Linux.

\myparaitemize{Generating the Datasets}
We first generate four types of sparse event tensors (1, 2, 3, and 4),
which have $100K$ observations as $\tensor \in \mathN^{100 \times 100 \times 100 \times 100}$.
Each \attribute of an event entry is drawn from multinomial distributions
whose parameter is defined by random values $[0.1, 0.5]$ and a Dirichlet prior.
Finally, 
four different synthetic datasets are built 
by using different combinations of event tensors as follows \cite{hallac2017toeplitz}:
``1,2,1'', ``1,2,3,2,1'', ``1,2,3,4,1,2,3,4'', ``1,2,2,1,3,3,3,1''.

\myparaitemize{Implementaion \& Parameters}
We used the open-source implementation of 
LDA, K-means, LOF, and iForest in \cite{pedregosa2011scikit}.
For NTM, we implemented it based on the pytorch framework 
and applied Adam optimization with a learning rate of $1e-3$,
following the design and the parameter setting in \cite{WangLCKLS19}.
We also used open-sourced implementations of 
TICC \cite{hallac2017toeplitz}, 
T-LSTM \cite{DBLP:conf/kdd/BaytasXZWJZ17},
RRCF \cite{10.5555/3045390.3045676}, 
CubeMarker \cite{honda2019multi},
and MemStream \cite{DBLP:conf/www/0001JSKH22}, 
provided by the authors,
following parameter settings as suggested in the original papers. 
For a fair comparison in terms of computational time, 
we implemented TriMine in Python,
following C implementation provided by the authors.
The input for LDA/NTM is bag-of-words representations of all the categories,
i.e., $\mathbf{W} \in \mathN^{\tau \times (\nunits_1+ \cdots + \nunits_\nmode)}$.
In evaluation of clustering accuracy,
the width of a current tensor is set with $10$.
Since TICC and T-LSTM need to specify the number of clusters, 
we set the true number of clusters.
DBSTREAM, which is implemented in \cite{river}, and \method 
are automatically determine the number of clusters.
We set the radius of each micro-clusters as $8.5$ for DBSTREAM,
and the number of \topics $\ntopic = 8$ for \method.
To validate detection accuracy, we set $\tau=1$ for all methods.
We used a $5\tau$ length of the stream to conduct the model initialization for \method.

\subsection{Effectiveness}
\label{sec:app:expeffect}
We also demonstrate how effectively \method works on the \electronics dataset.

\myparaitemize{Online Marketing Analytics}
\autoref{fig:electronics} shows stream mining results 
of \method on the \electronics.
This data is the purchase data obtained over a year 
from a large home appliances and electronics online store.
The data contains a list of two attributes; 867 brands and 124 item categories with an hourly timestamp.
\bit
\item
\myiparapara{Regime identification}
\method discovered four type of regimes in \autoref{fig:electronics} (a).
Specifically, 
our method found 
Regime \#2 during a short period around July,
and then 
discovered Regime \#3 for the summer season and Regime \#4 for the fall season.
This result shows that the behaviors of purchases shift 
with the transition of seasons.
%
\item
\myiparapara{Multi-aspect \topic analysis}
\autoref{fig:electronics} (b) shows
the temporal evolution of two major \topics ,
which are shown in the time \topic matrix $\Mtime$ in Regime~\#1.
We manually named the two \topics ``White goods'' and ``Small/Kitchen''.
These \topics exhibit contrasting behavior. 
The White goods sequence peaks at night,
while the Small/Kitchen sequence peaks during the daytime.
\hide{
In practice,
this discovery leads to more effective advertisements,
which highlight items depending on the times at which they are most popular.
Thanks to the automatic regime detection,
it can also adaptively switch the advertisements 
even if the customer behavior changes with the seasons.
}
\autoref{fig:electronics} (c) shows 
the \attribute \topic matrices $\{\Matt^{(\lmode)}\}_{\lmode=1}^2$ in Regime $\#1$,
namely the latent relationships between two \topics (row) 
and two \attributes (column).
\autoref{fig:electronics} (d) shows 
the changes of popular item categories in association with regime transitions.
These changes make sense.
As shown in the top figure, the White goods \topic in Regime \#2 (early summer)
has the strongest relationship with an air conditioner.
Similarly,
the bottom figure shows the Small/Kitchen \topic in Regime \#4 (Fall),
where an air heater appeared as a popular item.
\eit

%% file: TABLE/table_symbol.tex
\begin{table}[t]
\centering
\footnotesize
\caption{Symbols and definitions.}
\label{table:define}
\vspace{-1.5em}
\begin{tabular}{l|l}
\toprule
Symbol & Definition \\
\midrule
$\nmode $
    & Number of attributes of complex event tensor\\
$\nunits_1 \dots \nunits_{\nmode}$
    & A set consisting of a number of unique \units in each \attribute\\
${\duration}$
    & Length of whole tensor stream\\
${\tau}$
    & Length of current tensor\\
$\tensor$
    & Whole event tensor stream,

    $\tensor\shapeN^{\nunits_1 \times \dots \times \nunits_{\nmode} \times \duration}$\\
$\tensorC$
    & Current event tensor,
    $\tensorC\shapeN^{\nunits_1 \times \dots \times \nunits_{\nmode} \times \tau}$\\
\midrule
${\ntopic}$
    & Number of latent \topics \\
${\Matt^{(\lmode)}}$
    & $\lmode$-th \attribute \topic matrix, ${\ntopic}\times{\nunits_\lmode}$\\
${\Mtime}$
    & Time \topic matrix, ${\tau}\times{\ntopic}$\\
$Q$
    & FIFO queue for retaining past component matrices \\
$L$
    & The size of queue $Q$ \\
\midrule
$\nregime$ 
& Number of regimes\\
$\regime_\lregime$
    & $\lregime$-th regime parameter,
    i.e., $\regime_\lregime =\{\Matt^{(1)}, \ldots, \Matt^{(\nmode)}, \Mtime \}$\\
$\regimeset$
    & Regime set,
    i.e., $\regimeset =\{\regime_1, \ldots, \regime_\nregime\}$\\
$\nshiftp$ 
& Number of regime assignments (i.e., segments)\\
$\regimeassign_{\lshiftp} $
    & Trajectory of shift to $\lregime$-th regime at time $t_s$, i.e., $\regimeassign_{\lshiftp} = (t_s,\lregime)$ \\
$\regimeassignment$
    & Regime assignments, 
    i.e., $\regimeassignment =\{\regimeassign_1, \ldots, \regimeassign_\nshiftp \}$\\
$|\segmentset^{-1}_{\lregime}|$
    &Total segment length of the regime $\regime_\lregime$
    \\
\midrule
$\cand$
    & Compact description,
    i.e., $\cand =\{ \nregime, \regimeset, \nshiftp, \regimeassignment \}$ \\
$\score{\tensorC}$
    & Anomalousness score of $\tensorC$ \\
\bottomrule
\end{tabular}
\normalsize
\vspace{-1em}
\end{table}